\documentclass[journal]{IEEEtran}
\IEEEoverridecommandlockouts
\usepackage{cite}
\usepackage{amsmath,amssymb,amsfonts}
\usepackage{graphicx}
\usepackage{textcomp}
\usepackage{hyperref}       
\usepackage{url}            
\usepackage{booktabs}       
\usepackage{nicefrac}       
\usepackage{microtype}      
\usepackage{xcolor}         
\usepackage{etoolbox}
\usepackage{changepage}
\usepackage{titlesec}
\usepackage{algorithm}
\usepackage{algorithmic}
\usepackage{times}
\usepackage{float}
\usepackage{ragged2e}
\usepackage{subfigure}
\usepackage{verbatim}
\usepackage{diagbox}
\usepackage{threeparttable}
\usepackage{multirow}
\usepackage{multicol}
\usepackage{blindtext}
\usepackage[T1]{fontenc}

\newtheorem{definition}{Definition}
\newtheorem{proposition}{Proposition}

\newtheorem{proof}{Proof}

\def\BibTeX{{\rm B\kern-.05em{\sc i\kern-.025em b}\kern-.08em
    T\kern-.1667em\lower.7ex\hbox{E}\kern-.125emX}}
    
\begin{document}

\title{Preserving Temporal Dynamics \\in Time Series Generation}
\author{Ci Lin, Futong Li, Tet Yeap, and Iluju Kiringa}

\maketitle

\begin{abstract}
	Time-series data augmentation plays a crucial role in regression-oriented forecasting tasks, where limited data restricts the performance of deep learning models. While Generative Adversarial Networks (GANs) have shown promise in synthetic time-series generation, existing approaches primarily focus on matching marginal data distributions and often overlook the temporal dynamics that naturally exist in the original multivariate time series. When generating multivariate time series, this mismatch leads to distribution shift and temporal drift, thereby degrading the fidelity of the synthetic sequences.

	In this work, we propose a model-agnostic Markov Chain Monte Carlo (MCMC)-based framework to mitigate distribution shift and preserve temporal dynamics in synthetic time series. We provide a theoretical analysis of how conditional generative models accumulate deviations under sequential generation and demonstrate that the MCMC algorithm can correct these discrepancies by enforcing consistency with empirical transition statistics between neighboring time points.
	
	Extensive experiments on the Lorenz, Licor, ETTh, and ILI datasets using RCGAN, GCWGAN, TimeGAN, SigCWGAN, and AECGAN demonstrate that the proposed MCMC framework consistently improves autocorrelation alignment, skewness error, kurtosis error, R$^2$, discriminative score, and predictive score. These results suggest that synthetic time series consistent with the original data require explicit preservation of transition laws rather than solely relying on adversarial distribution matching, thereby offering a principled direction for improving generative modeling of time-series data.
\end{abstract}


\section{Introduction}\label{sec:introduction}

Time series analysis plays a central role in numerous real-world applications, including stock market prediction, water consumption monitoring, electrical signal analysis, weather forecasting, and nitrous oxide (N$_2$O) emission modeling \cite{altunkaynak2017monthly, idrees2019prediction, cui2020stacked, Lin2022Stacked}. In recent years, substantial progress has been achieved in time-series forecasting with the development of advanced deep learning architectures such as Temporal Convolutional Networks (TCN) \cite{lea2017temporal} and Transformer-based models \cite{zeng2023transformers}. These models demonstrate strong capability in capturing long-range dependencies and complex temporal patterns.

However, the effectiveness of high-capacity deep learning models critically depends on the availability of large-scale, high-quality training data \cite{lim2021time}. In many real-world time-series applications, collecting sufficient data is expensive, time-consuming, or even impractical \cite{duong2021review, lin2024agriculture}. This data scarcity issue has become a major bottleneck limiting the performance and generalization ability of deep forecasting models.To address this challenge, data augmentation and synthetic data generation have emerged as promising solutions \cite{demir2021data}. In particular, deep generative models have demonstrated remarkable success in learning complex data distributions and producing realistic synthetic samples \cite{brophy2023generative}.

Among generative approaches, Generative Adversarial Networks (GANs) \cite{goodfellow2014generative} have shown exceptional performance in high-dimensional data generation, particularly in computer vision tasks \cite{elanwar2025generative}. Motivated by these advances, significant efforts have been devoted to extending GAN frameworks to time-series generation \cite{brophy2023generative}. Early attempts such as Recurrent Conditional GAN (RCGAN) \cite{festag2023medical} and QuantGAN \cite{wiese2020quant} replace convolutional architectures with recurrent neural networks (e.g., LSTM) or TCN modules to better capture temporal dependencies.  TimeGAN \cite{yoon2019time} further incorporates supervised objectives and explicitly models temporal transitions $p(x_t \mid x_{1:t-1})$ in a latent representation space, thereby improving temporal coherence. Furthermore, SigCWGAN \cite{liao2020conditional} constructs a theoretically grounded conditional generative model that captures autoregressive structures through signature-based Wasserstein distances.

Despite these encouraging developments, generating long time-series sequences in an autoregressive manner remains fundamentally challenging. In particular, sequential generation often suffers from distribution shift and error accumulation, where small deviations at early time steps are progressively amplified during generation \cite{wang2023aec}. Such distributional discrepancies can significantly degrade temporal consistency and predictive reliability, especially in regression-oriented tasks where preserving conditional dynamics is crucial.

\subsection{Motivations and Contributions}

Despite recent advances in GAN-based time-series generation, most existing methods primarily focus on matching the marginal distribution of observed sequences. However, in regression-oriented time-series modeling, preserving the underlying temporal dynamics $p(x_{t+1} \mid x_{1:t})$ is fundamentally more important than reproducing static distributional statistics. 

Autoregressive generation inherently induces distribution shift: during sequential generation, the model conditions on its own predictions rather than true observations. Even small deviations at early time steps may be amplified over long horizons, leading to accumulated errors and degradation of temporal consistency. While prior works attempt to mitigate this issue through architectural modifications or auxiliary supervised objectives, there remains a lack of explicit mechanisms that enforce dynamical consistency during generation.

To address this limitation, we propose a Markov Chain Monte Carlo (MCMC) framework for conditional generative adversarial networks (CGANs) that explicitly corrects distributional and dynamical deviations in generated time-series trajectories. Our contributions are summarized as follows:

\begin{itemize}
	\item We reformulate time-series augmentation as the problem of preserving temporal dynamics rather than merely matching marginal data distributions.
	\item We provide a theoretical analysis showing how distribution shift arises in conditional generative models during autoregressive generation.
	\item We propose a novel MCMC-based correction module that enforces consistency with empirical temporal transition statistics through a Markov chain satisfying detailed balance.
	\item We demonstrate that the proposed framework consistently improves temporal fidelity and predictive performance across multiple benchmark datasets, including Lorenz, Licor, Electricity Transformer Temperature (ETTh), and Influenza-Like Illness (ILI), and state-of-the-art GAN architectures, such as RCGAN, RGWGAN, TimeGAN, SigCWGAN, and AECGAN.
\end{itemize}

\subsection{Organization}

The remainder of this paper is organized as follows. Section \ref{sec:literature_review} reviews related work on the two major approaches to augmenting multivariate time series: traditional machine learning methods and GAN-based approaches. Section \ref{sec:preliminary} introduces the preliminary knowledge of CGANs for time-series augmentation, provides a rigorous mathematical analysis of the distribution shift that exists in time-series generation, and explains the conditions under which MCMC can effectively preserve temporal dynamics in synthetic time series. Section \ref{sec:problem_formulation} presents the proposed MCMC correction framework and its algorithmic details. Section \ref{sec:experiment_discussion} reports extensive experimental evaluations across multiple datasets and GAN architectures. Finally, Section \ref{sec:conclusion_future_work} concludes the paper and discusses future research directions.

\section{Literature Review}\label{sec:literature_review}

\subsection{Machine Learning Methods for Time Series Data Augmentation}

Table \ref{table:data_augmentaion_time_series} provides an overview of common data augmentation methods for time series data. Similar to data augmentation techniques employed in computer vision, most time series data augmentation approaches rely on random transformations, including operations like cropping \cite{takahashi2019data}, scaling \cite{rashid2019window}, and rotation \cite{okafor2018analysis}. The difficulty in implementing random transformation-based data augmentation arises from the inherent diversity among time series datasets. Not all transformations are suitable for every time series dataset. For instance, the application of jittering (adding noise) assumes that it is typical for the time series patterns in a specific dataset to include noise. While this assumption is valid for sensor data, audio recordings, or Electroencephalogram (EEG) data, it may not be applicable to time series data obtained from object contours.

\begin{table}[htbp]
	\centering
	\scriptsize
	\begin{threeparttable}
		\caption{Machine Learning Approaches for Time Series Data Augmentation}
		\begin{tabular}
			{m{2cm}m{6cm}}
			\toprule Family & Methods\\
			\midrule Random Transformation & Jittering \cite{um2017data, bishop1995training, an1996effects}, Rotation \cite{um2017data}, Scaling \cite{rashid2019window},  Magnitude Warping, Flipping \cite{rashid2019window}, Permutation, Slicing, Time Warping, Time Masking, Frequency Masking, Frequency Warping, Fourier Transform \\
			\midrule Pattern Mixing & Deviation From Mean (DFM) \cite{yeomans2019simulating} , Interpolation \cite{chawla2002smote}, Time Aligned Averaging \cite{petitjean2011global}, Guided Warping \cite{iwana2021time}, Equalized Mixture Data Augmentation (EMDA) \cite{takahashi2017aenet}, Stochastic Feature Mapping (SFM) \cite{cui2015data} \\
			
			\midrule Decomposition & Seasonal-Trend decomposition using Loess (STL) \cite{bergmeir2016bagging}, Independent Component Analysis (ICA) \cite{comon1994independent}, Empirical Mode Decomposition(EMD) \cite{huang1998empirical} \\
			
			\midrule Generative Models & Gaussian Trees \cite{cao2014parsimonious}, Posterior Sampling \cite{fruhwirth1994data}, Linear Model \cite{meng1999seeking}, Markov Chain \cite{hou2018sequence}, Local and Global Trend (LGT) \cite{smyl2016data}, Generating Time Series (GRATIS) \cite{kang2020gratis}, Long Short Term Memory (LSTM) \cite{smyl2016data}, Generative Adversarial Network (GAN) \cite{madhu2019data, yoon2019time}, WaveNet \cite{wang2019speech}, Autoencoders \cite{tu2018spatial} \\
			\bottomrule
		\end{tabular}
		\label{table:data_augmentaion_time_series}
	\end{threeparttable}
\end{table}

In addition to random transformations, researchers have also explored the utilization of intrinsic dataset information for synthesizing time series data. This can be accomplished through methods such as pattern mixing, decomposition, and generative models. Pattern mixing encompasses the practice of blending two or more existing time series to create fresh, unique patterns. The fundamental idea behind this technique is that by amalgamating distinct existing patterns, new samples can be crafted, combining distinctive features from each contributing pattern. Some common approaches to pattern mixing include Deviation From Mean (DFM) \cite{yeomans2019simulating}, interpolation \cite{chawla2002smote}, and Equalized Mixture Data Augmentation (EMDA) \cite{takahashi2017aenet}. Decomposition methods extract features from the dataset, such as trend components \cite{bergmeir2016bagging} and independent components \cite{eltoft2002data}, to generate new patterns based on these extracted features. Generative models take a somewhat indirect approach by utilizing the feature distributions within the dataset to produce new patterns. More recently, generative models employing neural networks, such as GANs \cite{goodfellow2014generative}, have gained prominence.

\subsection{Generative Adversarial Network for Time Series Data Augmentation}

GANs represent one of the most significant advancements in generative technology. Since their initial introduction in a seminal research paper \cite{goodfellow2014generative}, these frameworks have been applied extensively to generate synthetic sequences in diverse domains, including renewable scenarios \cite{chen2018model}, sensor data \cite{alzantot2017sensegen}, finance \cite{simonetto2018generating}, biosignals \cite{haradal2018biosignal}, smart grid data \cite{zhang2018generative}, and text \cite{zhang2016generating}. Recently, researchers have shifted their focus to synthesizing time series data, leading to the development of several GANs specifically designed for this purpose, as shown in Table \ref{table:data_augmentaion_gan}.

\begin{table}[htbp]
	\centering
	\scriptsize
	\begin{threeparttable}
		\caption{Generative Adversarial Network for Time Series Data Augmentation}
		\begin{tabular}
			{m{2cm}m{6cm}}
			\toprule Models & Characteristic\\
			\midrule WGAN \cite{arjovsky2017wasserstein, gulrajani2017improved} & Enhance the stability of learning, mitigate issues such as mode collapse, and offer meaningful learning curves that can aid in debugging and optimizing hyperparameters. \\
			\midrule C-RNN-GAN \cite{mogren2016c} & Work with continuous sequential data and apply it through training on a collection of classical music.  \\
			\midrule RGAN \cite{fekri2019generating} & Replace CNNs from image GANs with recurrent neural networks (RNNs) because of RNN's ability to capture temporal dependencies. \\
			\midrule RCGAN \cite{esteban2017real} & Utilize RNNs in both the generator and the discriminator, with both models being conditioned on auxiliary information. \\
			\midrule WaveGAN \cite{donahue2018adversarial} &  Capable of generating one-second audio waveform segments with global coherence; suitable for sound effect creation. \\
			\midrule TimeGAN \cite{yoon2019time} & Generate authentic time series data that blends the flexibility of the unsupervised approach with the control offered by supervised training. \\
			\midrule GT-GAN \cite{jeon2022gt} & Capable of generating both regular and irregular time series data. \\
			\midrule SigCWGAN \cite{liao2020conditional} & a stable, theoretically grounded conditional GAN for time-series generation that replaces adversarial discriminators with analytic Wasserstein distances in signature space, enabling efficient learning of high-order temporal dependencies. \\
			\midrule AEC-GAN \cite{wang2023aec} & AEC-GAN refines generated sequences (regarded as the conditioning variables) for high-quality long sequence generation auto-regressively. \\
			\bottomrule
		\end{tabular}
		\label{table:data_augmentaion_gan}
	\end{threeparttable}
\end{table}

The earliest model for generating time series data, C-RNN-GAN \cite{mogren2016c} and RGAN \cite{fekri2019generating} employs a standard GAN framework adapted for sequential data by utilizing RNN in both its generator and discriminator. Furthermore, the Recurrent Conditional GAN (RCGAN) \cite{esteban2017real} takes a similar approach but introduces minor architectural differences, such as removing the dependence on previous output while conditioning on additional input \cite{mirza2014conditional}. WaveGAN \cite{donahue2018adversarial} generates time series data by estimating the conditional probability of the preceding data using dilated causal convolution. TimeGAN \cite{yoon2019time} presents a framework that combines adversarial training of GANs with supervised training to predict $x_{i+1}$ from $x_i$, where $x_i$ and $x_{i+1}$ denote two multivariate time series values at time $t_i$ and $t_{i+1}$, respectively. Building upon TimeGAN, GT-GANs \cite{jeon2022gt} are capable of synthesizing both regular and irregular time series data. However, when dealing with time series data collected from dynamical systems, capturing the underlying dynamic behavior using these methods presents a significant challenge. This challenge arises because the dynamic behavior is implicitly encoded in the first-order difference or the first-order derivative over time.

\section{Preliminary Knowledge}\label{sec:preliminary}

\subsection{Conditional Generative Adversarial Networks for Time Series Augmentation}

Let $\{X_t\}_{t=1}^T \subset \mathbb{R}^d$ denote a multivariate time series. 
For time-series generation and forecasting, it is often more natural to model the
\emph{conditional distribution of future trajectories given past observations},
rather than the joint distribution over the entire sequence \cite{liao2020conditional}.

Using a sliding-window strategy, we define for each time index $t$:
\begin{equation}
	\begin{aligned}
	&\ X^{\text{past}}_t := (X_{t-p+1}, \dots, X_t) \in \mathbb{R}^{d \times p}, \\
	&\ X^{\text{future}}_t := (X_{t+1}, \dots, X_{t+q}) \in \mathbb{R}^{d \times q},
	\end{aligned}
\end{equation}
where $p$ and $q$ denote the lengths of the past and future windows, respectively.
The objective of conditional time-series generation is to learn the conditional law
\begin{equation}
	\mu(x) := \mathcal{L}\!\left( X^{\text{future}} \mid X^{\text{past}} = x \right).
\end{equation}

CGANs address this problem by conditioning both the generator and the discriminator on the observed past sequence \cite{mirza2014conditional}. 

In this context, let $X^{\text{past}} \in \mathbb{R}^{d \times p}$ denote the conditioning past window and
$X^{\text{future}} \in \mathbb{R}^{d \times q}$ denote the target future window. A conditional
generator $G_\theta$ produces a synthetic future sequence given the past and a latent noise vector
$z \sim p_z$:
\begin{equation}
	\hat{X}^{\text{future}} = G_\theta\!\left(X^{\text{past}}, z\right).
\end{equation}
A conditional discriminator $D_\phi$ takes as input a pair $(X^{\text{past}}, X^{\text{future}})$
and outputs the probability that the future window is real (rather than generated).The standard conditional GAN objective is formulated as the following minimax problem:
\begin{equation}
	\begin{aligned}
		& \min_{\theta}\max_{\phi}\ \mathcal{L}_{\text{cGAN}}(\theta,\phi) \\
		& = \mathbb{E}_{(X^{\text{past}},X^{\text{future}})\sim p_{\text{data}}}
		\big[\log D_\phi(X^{\text{past}}, X^{\text{future}})\big] \\
		& + \mathbb{E}_{X^{\text{past}}\sim p_{\text{data}},\ z\sim p_z}
		\big[\log\big(1-D_\phi(X^{\text{past}}, G_\theta(X^{\text{past}}, z))\big)\big].
		\label{eq:cgan_minimax}
	\end{aligned}
\end{equation}

In the time-series setting, conditional generators are often designed to respect an autoregressive structure, where future values are generated recursively based on past observations and stochastic innovations:
\begin{equation}
	X_{t+1} = g\!\left( X^{\text{past}}_t, \varepsilon_{t+1} \right),
\end{equation}
with $\varepsilon_t$ denoting independent noise variables. Multi-step future
trajectories are then obtained by iteratively applying the conditional generator in
a rolling-window fashion.

Compared with unconditional GANs that attempt to model the full joint distribution
$p(X_{1:T})$, CGANs focus on local temporal transitions
$p(X^{\text{future}} \mid X^{\text{past}})$. This formulation significantly reduces
the dimensionality of the learning problem, improves stability for long sequences,
and naturally aligns with forecasting and simulation tasks in dynamical systems.

\subsection{Distribution Shift in Generative Adversarial Networks}

GANs commonly suffer from distribution shift, as adversarial training under finite samples and limited model capacity does not guarantee full recovery of the data distribution. Prior work has shown that this manifests as mode collapse and boundary distortion, even when generated samples appear realistic \cite{arora2017generalization, santurkar2018classification, adler2018banach}. In time-series settings, such distributional mismatches are especially problematic due to accumulated errors in temporal dependencies \cite{wang2023aec}.

Generative models are typically trained on a finite set of samples drawn from an unknown data distribution, and thus operate on the induced empirical distribution rather than the true population distribution. This distinction is central to the analysis of distributional generalization in adversarial learning. We begin by defining the empirical distribution.

\begin{definition}
	Let $\mu$ be a probability distribution on $\mathbb{R}^d$, and let
	$\{x_i\}_{i=1}^{m}$ be $m$ independent samples drawn from $\mu$.
	The empirical distribution associated with these samples is defined as
	\[
	\hat{\mu} = \frac{1}{m} \sum_{i=1}^{m} \delta_{x_i},
	\]
	where $\delta_{x_i}$ denotes the Dirac measure centered at $x_i$.
\end{definition}

\begin{proposition}
	Let $\mu = \mathcal{N}(0, \sigma^{2} I_d)$ be a Gaussian distribution on
	$\mathbb{R}^d$, and let $\hat{\mu}$ denote the empirical distribution constructed
	from $m$ i.i.d.\ samples drawn from $\mu$. Then,
	\[
	d_{\mathrm{JS}}(\mu, \hat{\mu}) = \log 2,
	\qquad
	d_{\mathrm{W}}(\mu, \hat{\mu}) \ge 1.1,
	\]
	where $d_{\mathrm{JS}}(\cdot,\cdot)$ denotes the Jensen--Shannon divergence,
	$d_{\mathrm{W}}(\cdot,\cdot)$ denotes the Wasserstein distance, and
	$c > 0$ is a constant independent of $m$.
\end{proposition}

\begin{proof}
	For the Jensen--Shannon divergence, observe that $\mu$ is absolutely continuous
	with respect to Lebesgue measure, while $\hat{\mu}$ is a discrete distribution
	supported on a finite set. Consequently, the two measures are mutually singular,
	and the Jensen--Shannon divergence attains its maximal value,
	$d_{\mathrm{JS}}(\mu, \hat{\mu}) = \log 2$.
	
	For the Wasserstein distance, let $x_1, x_2, \ldots, x_m$ denote the empirical
	samples, which are fixed arbitrarily. Let $y \sim \mathcal{N}(0, d^{-1} I_d)$.
	By standard concentration inequalities for Gaussian random vectors and a union
	bound over the $m$ sample points, we have
	\[
	\begin{aligned}
	& \Pr\!\left[ \forall i \in [m], \; \|y - x_i\| \ge 1.2 \right] \\
	&\ \;\ge\; 1 - m \exp(-\Omega(d)) \;\ge\; 1 - o(1).
	\end{aligned}
	\]
	Using the earth-mover (optimal transport) interpretation of the Wasserstein
	distance, it follows that transporting mass from $\mu$ to $\hat{\mu}$ incurs a
	cost of at least $1.2$ whenever $y$ is at distance at least $1.2$ from all
	$x_i$. Therefore,
	\[
	d_{\mathrm{W}}(\mu, \hat{\mu})
	\;\ge\;
	1.2 \cdot \Pr\!\left[ \forall i \in [m], \; \|y - x_i\| \ge 1.2 \right]
	\;\ge\; 1.1.
	\]
\end{proof}

In addition to unconditional generative modeling, distribution shift arises naturally in CGANs, which are widely used for regression-based generation and time series modeling. In a CGAN,
the generator is trained to model a conditional distribution of the form $p(y \mid x)$, where the conditioning variable $x$ is drawn from the data distribution during training.

However, during sequential generation, the conditioning variable may itself be generated by the model or sampled from a distribution that differs from the original data distribution. This mismatch induces a distribution shift in the generated outputs, even if the conditional generator is accurately learned.

\begin{proposition}[Distribution shift in conditional GANs]
	\label{prop:cgan_shift_rigorous}
	Let $(X,Y)$ be random variables on measurable spaces $(\mathcal{X},\mathcal{B}_{\mathcal{X}})$
	and $(\mathcal{Y},\mathcal{B}_{\mathcal{Y}})$ with joint distribution
	$p_{\mathrm{data}}(x,y)$, marginal $p_{\mathrm{data}}(x)$, and conditional
	$p_{\mathrm{data}}(y\mid x)$. Let $p_\theta(y\mid x)$ be a learned conditional
	distribution satisfying
	$p_\theta(\cdot \mid x)=p_{\mathrm{data}}(\cdot \mid x)$ for
	$p_{\mathrm{data}}$-almost every $x$.
	
	For any alternative conditioning distribution $q(x)$, define the induced output
	marginal
	\[
	q_\theta(B) := \int_{\mathcal{X}} p_\theta(B \mid x)\, q(dx),
	\quad B \in \mathcal{B}_{\mathcal{Y}},
	\]
	and the true output marginal
	\[
	p_{\mathrm{data}}(B) := \int_{\mathcal{X}} p_{\mathrm{data}}(B \mid x)\,
	p_{\mathrm{data}}(dx).
	\]
	If there exists a measurable set $B \in \mathcal{B}_{\mathcal{Y}}$ such that the
	function $g_B(x):=p_{\mathrm{data}}(B\mid x)$ satisfies
	\[
	\int_{\mathcal{X}} g_B(x)\, q(dx) \neq \int_{\mathcal{X}} g_B(x)\, p_{\mathrm{data}}(dx),
	\]
	then $q_\theta \neq p_{\mathrm{data}}$ as probability measures on $\mathcal{Y}$.
	Moreover,
	\[
	d_{\mathrm{TV}}(q_\theta, p_{\mathrm{data}})
	\;\ge\;
	\left|\int_{\mathcal{X}} g_B(x)\, q(dx) - \int_{\mathcal{X}} g_B(x)\, p_{\mathrm{data}}(dx)\right|.
	\]
\end{proposition}

\begin{proof}
	Since $p_\theta(\cdot\mid x)=p_{\mathrm{data}}(\cdot\mid x)$ for
	$p_{\mathrm{data}}$-a.e.\ $x$, we have $p_\theta(B\mid x)=p_{\mathrm{data}}(B\mid x)=g_B(x)$
	for $p_{\mathrm{data}}$-a.e.\ $x$. Therefore, for any measurable $B\subseteq\mathcal{Y}$,
	\[
	q_\theta(B)
	=
	\int_{\mathcal{X}} p_\theta(B\mid x)\,q(dx)
	=
	\int_{\mathcal{X}} g_B(x)\,q(dx),
	\]
	and similarly,
	\[
	p_{\mathrm{data}}(B)
	=
	\int_{\mathcal{X}} p_{\mathrm{data}}(B\mid x)\,p_{\mathrm{data}}(dx)
	=
	\int_{\mathcal{X}} g_B(x)\,p_{\mathrm{data}}(dx).
	\]
	By the assumed inequality of these two integrals, we obtain
	$q_\theta(B)\neq p_{\mathrm{data}}(B)$, hence $q_\theta \neq p_{\mathrm{data}}$ as measures.
	
	For the total variation bound, recall
	$d_{\mathrm{TV}}(\nu_1,\nu_2) = \sup_{A\in\mathcal{B}_{\mathcal{Y}}} |\nu_1(A)-\nu_2(A)|$.
	Thus,
	\[
	\begin{aligned}
	&\ d_{\mathrm{TV}}(q_\theta, p_{\mathrm{data}})
	\ge |q_\theta(B) - p_{\mathrm{data}}(B)| \\
	&\ = \left|\int_{\mathcal{X}} g_B(x)\, q(dx) - \int_{\mathcal{X}} g_B(x)\, 	p_{\mathrm{data}}(dx)\right|.
	\end{aligned}
	\]
\end{proof}

From a dynamical systems perspective, this form of distribution shift can be interpreted as a perturbation of the initial condition. In time series conditional GANs, the conditioning variable represents the current system state, and future outputs are generated through repeated conditional transitions. Consequently, even small deviations in the conditioning distribution may be amplified by the underlying dynamics, leading to substantial divergence in future states.

\subsection{Markov Chain Monte Carlo}

MCMC methods sample from a target distribution by
constructing a Markov chain whose stationary distribution coincides with the
target.

Let $\mathcal{X}$ be a measurable state space. A Markov chain
$\{X_t\}_{t \ge 0}$ is defined by a transition kernel $P(x, A)$ such that
\[
\Pr(X_{t+1} \in A \mid X_t = x) = P(x, A),
\quad \forall A \subseteq \mathcal{X}.
\]

A probability measure $\pi$ on $\mathcal{X}$ is stationary if
\[
\pi(A) = \int_{\mathcal{X}} P(x, A)\,\pi(dx),
\quad \forall A \subseteq \mathcal{X}.
\]

A sufficient condition for stationarity is the detailed balance condition,
\[
\pi(x)\, P(x, dx') = \pi(x')\, P(x', dx),
\quad \forall x, x' \in \mathcal{X}.
\]

The Metropolis--Hastings algorithm constructs such a transition kernel. Given a
current state $x$, a proposal $x'$ is drawn from a proposal distribution
$q(x' \mid x)$ and accepted with probability
\[
\alpha(x, x') = \min\!\left(1,
\frac{\pi(x')\, q(x \mid x')}
{\pi(x)\, q(x' \mid x)}
\right).
\]
If the proposal is rejected, the Markov chain remains at $x$. Under standard
conditions on the proposal distribution $q$, the resulting Markov chain admits
$\pi$ as its stationary distribution.

\begin{proposition}
	If a Markov transition kernel $P$ satisfies detailed balance with respect to
	$\pi$, then $\pi$ is a stationary distribution of the Markov chain.
\end{proposition}

\begin{proof}
	For any measurable set $A \subseteq \mathcal{X}$,
	\[
	\int_{\mathcal{X}} P(x,A)\,\pi(dx)
	= \int_{\mathcal{X}} \int_{\mathcal{X}} \mathbf{1}_A(x')\,\pi(dx)\,P(x,dx').
	\]
	Applying detailed balance and integrating out $x$ yields
	\[
	\int_{\mathcal{X}} \mathbf{1}_A(x')\,\pi(dx') = \pi(A),
	\]
	which proves stationarity.
\end{proof}

\begin{proposition}[Stationary and reversible Markov chains satisfy detailed balance]
	Let $\{X_t\}$ be a time-homogeneous Markov chain with transition kernel $P$.
	Assume it is stationary with marginal distribution $\pi$ and reversible, i.e.,
	$(X_t,X_{t+1}) \stackrel{d}{=} (X_{t+1},X_t)$. Then $P$ satisfies detailed
	balance with respect to $\pi$:
	\[
	\pi(dx)\,P(x,dy) = \pi(dy)\,P(y,dx).
	\]
\end{proposition}

\begin{proof}
	Define the joint law of two consecutive states under stationarity:
	\[
	\eta(dx,dy) := \mathbb{P}(X_t \in dx, X_{t+1} \in dy).
	\]
	By the Markov property and stationarity, this joint law factorizes as
	\[
	\begin{aligned}
	\eta(dx,dy) &\ = \mathbb{P}(X_t \in dx)\,\mathbb{P}(X_{t+1}\in dy \mid X_t=x) \\
	&\ = \pi(dx)\,P(x,dy).
	\end{aligned}
	\]
	Reversibility implies $\eta(dx,dy)=\eta(dy,dx)$. Therefore,
	\[
	\pi(dx)\,P(x,dy) = \eta(dx,dy) = \eta(dy,dx) = \pi(dy)\,P(y,dx),
	\]
	which is exactly the detailed balance condition.
\end{proof}

\section{Problem Formulation}\label{sec:problem_formulation}

Let $\mathcal{X} = (\mathbb{R}^d)^T$ denote the space of multivariate time series of length $T$ and dimension $d$. Each time series $x=(x_1,\ldots,x_T)\in\mathcal{X}$ is viewed as a finite trajectory generated by an unknown (possibly stochastic) dynamical system. Let $\mu$ denote the population distribution over $\mathcal{X}$ induced by this system, and let $\mathcal{D}=\{x_i\}_{i=1}^m$ be a finite dataset drawn i.i.d.\ from $\mu$, inducing the empirical distribution $\hat{\mu}$.

In time series augmentation, the primary objective is to preserve temporal dynamics rather than merely match static distributions. Temporal dynamics characterize the evolution of system states over time and can be formalized through conditional transition laws of the form $\mu(X_{t+1}\mid X_{1:t})$, $t=1,\ldots,T-1$. A synthetic distribution is said to preserve temporal dynamics if these conditional laws are well approximated.

Let $G_\theta$ be a generative model trained on $\hat{\mu}$, inducing a distribution $\mu_\theta$ over $\mathcal{X}$. Due to finite-sample effects and sequential generation, samples drawn from $\mu_\theta$ may deviate from $\mu$, particularly in their temporal dependence structure. We therefore seek a correction mechanism that improves temporal fidelity while remaining computationally tractable.

To this end, we construct a Markov chain on $\mathcal{X}$ with transition kernel $P$ and target distribution $\pi$, designed to encode desired statistical and dynamical properties. The kernel $P$ is constructed to satisfy detailed balance with respect to $\pi$, ensuring that $\pi$ is a stationary distribution of the Markov chain. Starting from an initial sample $X^{(0)}\sim\mu_\theta$, the refined sample $X^{(K)}$ induces a refined distribution $\tilde{\mu}_\theta$.

\subsection{An MCMC-Based Correction Framework for GANs}\label{sec:mcmc_framework}

\begin{figure}[htbp]
	\centering
	\includegraphics[width=0.8\linewidth]{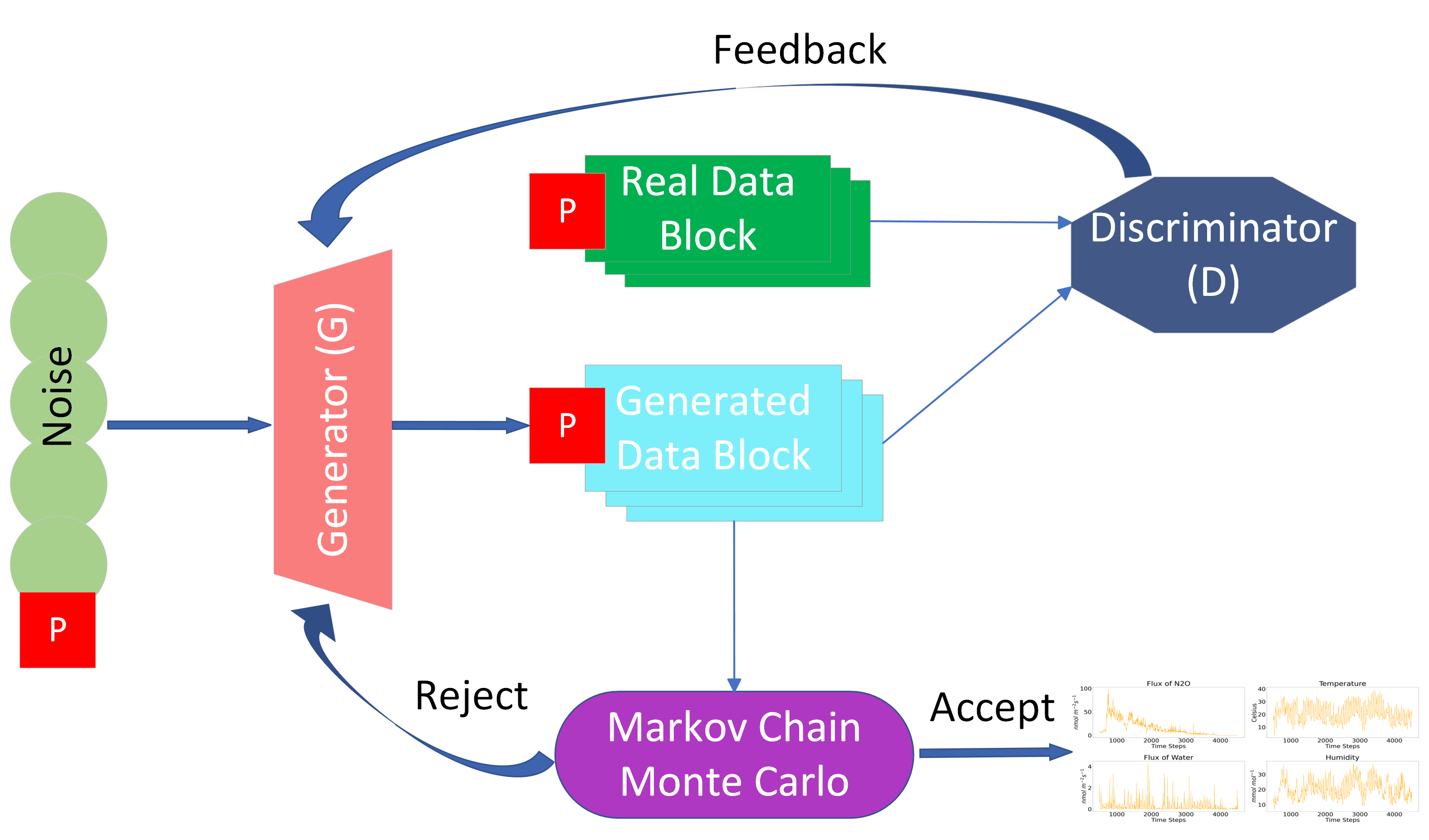}
	\caption{An MCMC-Based Correction Framework for GANs}
	\label{fig:architecture}
\end{figure}

As shown in Figure \ref{fig:architecture}, in general the proposed framework extends a conditional GAN by integrating a MCMC–based correction module to improve the temporal consistency of generated future time series. 

In this framework, We adopt a conditional GAN architecture, where generation is conditioned on a past time series segment $p$, following standard practice. The generator produces future time series blocks $q$ conditioned on $p$, and the discriminator distinguishes real and generated future blocks under the same conditioning \cite{mirza2014conditional, liao2020conditional, wang2023aec}.

In the MCMC module, a Metropolis–Hastings–type algorithm is used to select candidate samples generated by the generator. By constructing a transition rule that satisfies detailed balance with respect to a target distribution encoding desired temporal properties, the correction process corrects both distributional and dynamical deviations introduced by the generator, as shown in Algorithm \ref{alg:mcmc_algorithm}. The final output consists of MCMC-corrected future time series that better preserve statistical realism and temporal dynamics relative to the original data.

\subsection{Modified Metropolis-Hastings Algorithm}

\begin{algorithm}[htbp]
	\caption{Markov Chain Monte Carlo Module}
	\label{alg:mcmc_algorithm}
	\begin{algorithmic}[1]
		\STATE \textbf{Input:} A time series source $S$; a prior probability distribution $\pi(\theta)$; and a conditional $GAN$ model to generate candidate samples,a weighting factor $\beta$ regulating the relative influence of real and generated components.
		\STATE \textbf{Output:} A generated time series $glst$.
		\STATE \textbf{Initialization:} Initialize $\theta$ = $S[1] - S[0]$; an empty time series list $glst$; and a vector $v$ to store the last item added into candidate time series.
		\WHILE{LENGTH($glst$) != LENGTH($S$)}
		\STATE Draw the actual sample $s_{t} = S(t)$ 
		\WHILE{new sample is not added to $glst$}
		\STATE Generate a candidate sample $q_{t+1} = GAN(Z, p_t)$
		\IF{the length of $glst$ is zero}
		\STATE $\theta' \leftarrow q_{0} - s_{0}$
		\ELSE
		\STATE $\theta' \leftarrow (1 - \beta) (q_{t+1} - v) + \beta (q_{t+1} - s_{t})$
		\ENDIF
		\STATE $\gamma \leftarrow \min\!\left\{\dfrac{\pi(\theta')}{\pi(\theta) + \epsilon},\, 1 \right\}$
		\STATE Draw $u \sim \text{Uniform}(0,1)$
		\IF{$u \le \gamma$}
		\STATE $\theta \leftarrow \theta'$
		\STATE $v \leftarrow q_{t+1}$
		\STATE Add the selected sample $q_{t+1}$ to the time series list $tlst$
		\STATE \textbf{break}
		\ENDIF
		\ENDWHILE
		\ENDWHILE
		\RETURN time series $glst$
	\end{algorithmic}
\end{algorithm}

Algorithm \ref{alg:mcmc_algorithm} describes an MCMC-based correction procedure that constructs a synthetic time series sequentially. At each timestamp, the generator proposes candidate future values conditioned on the past, while an MCMC filter selects proposals whose local temporal changes are consistent with empirical dynamics observed in the real data.

The Markov chain is defined on an auxiliary state space $\Theta$, where each state
\[
\theta_t = x_t - x_{t-1}
\]
represents a first-order difference of the time series. The target distribution $\pi(\theta)$ is defined as the empirical joint probability distribution of first-order differences estimated from the real time series. This distribution captures local temporal dynamics and serves as the reference distribution guiding the correction process.

At time step $t$, given the current state $\theta$, a conditional GAN generates a
candidate future value
\[
q_{t+1} \sim GAN(\cdot \mid p_t),
\]
where $p_t$ denotes the past time series used for conditioning. Rather than proposing $\theta'$ directly, the algorithm derives a candidate discrepancy variable $\theta'$ from the generated value $q_{t+1}$, the real observation $s_t$, and the previously accepted synthetic value $v$.

Specifically, the candidate discrepancy is computed as
\[
\theta' \leftarrow (1-\beta)(q_{t+1} - v) + \beta(q_{t+1} - s_t),
\]
with a special initialization step when no synthetic value has yet been accepted.

Following the Metropolis--Hastings principle, the candidate update is evaluated using the target distribution $\pi(\theta)$. The transition from the current state $\theta$ to the candidate state $\theta'$ is accepted with probability
\[
\alpha(\theta,\theta') =
\min\!\left\{
1,\,
\frac{\pi(\theta')}{\pi(\theta) + \epsilon}
\right\}.
\]
A uniform random variable $u \sim \mathrm{Uniform}(0,1)$ is drawn, and the update is accepted if $u \le \alpha(\theta,\theta')$; otherwise, the state remains unchanged and a new candidate is generated.

Upon acceptance, the state is updated as $\theta \leftarrow \theta'$, the accepted value $q_{t+1}$ is appended to the generated time series, and the algorithm advances to the next timestamp. If the proposal is rejected, the GAN generates another candidate sample for evaluation.

\section{Experiments and Discussion}\label{sec:experiment_discussion}

\begin{table*}[htbp]
	\centering
	\small
	\begin{threeparttable}
		\caption{Summary of Evaluation Metrics}
		\label{tab:metrics_formula_summary}
		\begin{tabular}{m{2.8cm}m{6.5cm}>{\centering\arraybackslash}m{6.5cm}}
			\toprule
			\textbf{Metric} & \textbf{Description} & \textbf{Definition} \\
			\midrule
			
			ACF Error
			&
			Autocorrelation Function (ACF) measures the correlation between a time series and its lagged values, capturing temporal dependence and memory effects. The ACF error quantifies discrepancies in these correlations between real and generated data.
			&
			$\displaystyle
			\mathcal{L}_{\mathrm{ACF}}
			=
			\frac{1}{K}
			\sum_{k=1}^{K}
			\left|
			\mathrm{ACF}_{\text{real}}(k)
			-
			\mathrm{ACF}_{\text{gen}}(k)
			\right|
			$
			\\
			
			\midrule
			Skewness Error
			&
			Skewness measures the asymmetry of a data distribution around its mean. Preserving skewness ensures that generated data reproduces asymmetric behaviors such as long-tailed or bursty patterns.
			&
			$\displaystyle
			\mathcal{L}_{\mathrm{Skew}}
			=
			\left|
			\frac{\mathbb{E}[(x-\mu)^3]}{\sigma^3}
			-
			\frac{\mathbb{E}[(\hat{x}-\hat{\mu})^3]}{\hat{\sigma}^3}
			\right|
			$
			\\
			
			\midrule
			Kurtosis Error
			&
			Kurtosis measures the heaviness of distribution tails and the sharpness of its peak. Matching kurtosis is important for reproducing extreme events and heavy-tailed phenomena.
			&
			$\displaystyle
			\mathcal{L}_{\mathrm{Kurt}}
			=
			\left|
			\frac{\mathbb{E}[(x-\mu)^4]}{\sigma^4}
			-
			\frac{\mathbb{E}[(\hat{x}-\hat{\mu})^4]}{\hat{\sigma}^4}
			\right|
			$
			\\
			
			\midrule
			$R^2$
			&
			Measures the proportion of variance in the real time series that is explained by the generated data, reflecting global similarity between signals.
			&
			$\displaystyle
			R^2
			=
			1
			-
			\frac{
				\sum_{t}(y_t - \hat{y}_t)^2
			}{
				\sum_{t}(y_t - \bar{y})^2
			}
			$
			\\
			
			\midrule
			Discriminative Score
			&
			Evaluates how difficult it is for a trained classifier to distinguish real sequences from generated ones, reflecting the realism of synthetic data.
			&
			$\displaystyle
			\mathrm{DS}
			=
			\left|
			\mathrm{Acc}(D)
			-
			0.5
			\right|
			$
			\\
			
			\midrule
			Predictive Score
			&
			Assesses whether predictive models trained on synthetic data can generalize to real data, indicating preservation of temporal dynamics and conditional dependencies.
			&
			$\displaystyle
			\mathrm{PS}
			=
			\mathbb{E}
			\left[
			\| x_{t+1} - \hat{x}_{t+1} \|^2
			\right]
			$
			\\
			
			\bottomrule
		\end{tabular}
		
		\begin{tablenotes}[flushleft]
			\footnotesize
			\item Notes: $x$ and $\hat{x}$ denote real and generated samples, respectively; $\mu,\sigma$ and $\hat{\mu},\hat{\sigma}$ are the corresponding sample means and standard deviations; $D$ denotes a post-hoc discriminator trained to distinguish real and synthetic sequences; $K$ is the number of considered lags.
		\end{tablenotes}
		
	\end{threeparttable}
\end{table*}

\subsection{Experiments Setting}
\subsubsection{Dataset}
To evaluate the effectiveness of the proposed generative framework on both benchmark and real-world time-series data, we conduct experiments on four representative datasets: Lorenz, Electricity Transformer Temperature (ETTh), Influenza-Like Illness (ILI), and Agricultural N$_2$O Emission (also known as Licor). These datasets cover a wide range of temporal characteristics, including chaotic dynamics, long-term trends, seasonal patterns, and environmentally driven stochastic processes. All datasets are segmented into overlapping subsequences using a fixed context length and prediction horizon. Specifically, the conditioning window length is set to p=16 and the generation length is set to q=32. During training, the generator learns to produce a q-step sequence conditioned on the preceding p observations.

\subsubsection{Benchmarks}
To comprehensively evaluate the effectiveness of the proposed MCMC-enhanced generation framework, we conduct comparative experiments using five representative GANs for time series generation: RCGAN \cite{hyland2017real}, RCWGAN \cite{wiese2020quant}, SigCWGAN \cite{liao2020conditional}, TimeGAN \cite{yoon2019time}, and AEC-GAN \cite{wang2023aec}. To further reduce distribution shift, we integrate an MCMC-based correction module into each benchmark model. After the generator produces synthetic sequences, the MCMC module refines the generated samples by sampling from a corrected distribution conditioned on the original data statistics. This procedure allows the generated time series to better align with the target distribution while preserving temporal dynamics learned by the generative models.

\subsubsection{Measurement Metrics}
The evaluation metrics adopted in this study are summarized in Table~\ref{tab:metrics_formula_summary}. These metrics jointly assess the quality of generated time series from complementary perspectives, including temporal dependence (ACF), distributional characteristics (skewness error and kurtosis error), global similarity to real data ($R^2$), as well as realism (discriminative score) and functional utility  (predictive score). Together, they provide a comprehensive and quantitative evaluation of the extent to which the synthetic data preserves the statistical structure, temporal dynamics, and predictive behavior of the original time series.

\subsection{Autocorrelation Analysis}
\begin{figure}[htbp]	
	\centering
	\subfigure
	{\label{fig:aecgan_etth1_acf}}
	\addtocounter{subfigure}{-1}
	\subfigure[Autocorrelation of one dimension in the ETTh Dataset (AEC-GAN)]{
		\includegraphics[width=0.48\textwidth]{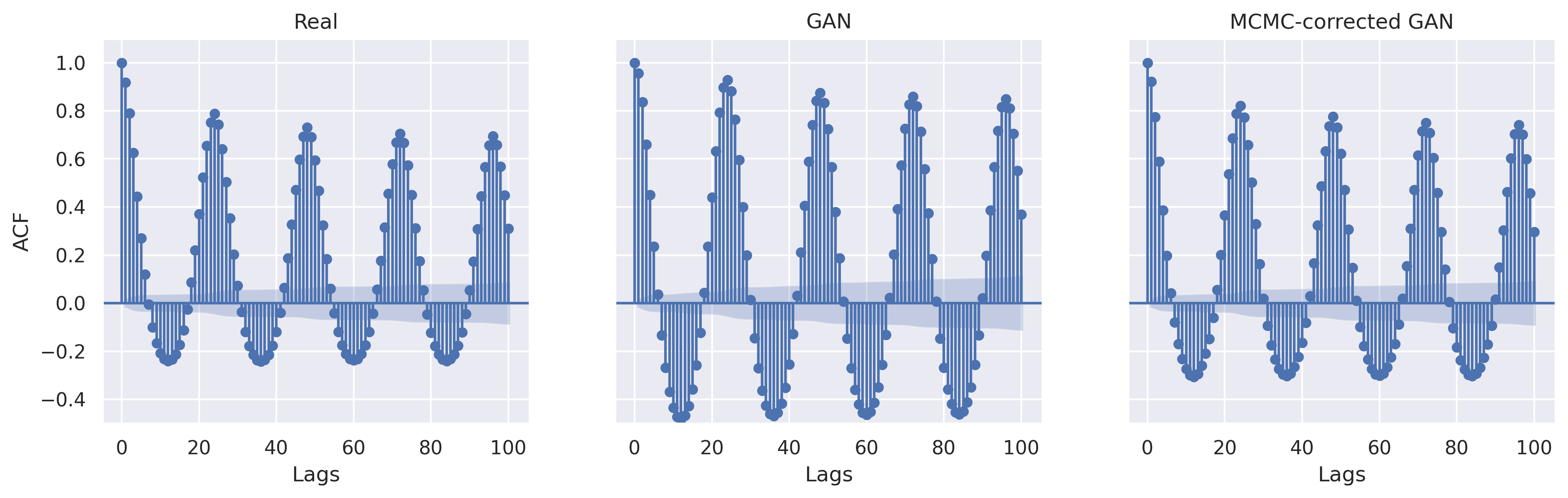}
	}
	\subfigure
	{\label{fig:rcgan_licor_acf}}
	\addtocounter{subfigure}{-1}
	\subfigure[Autocorrelation of one dimension in the Licor Dataset (RC-GAN)]{
		\includegraphics[width=0.48\textwidth]{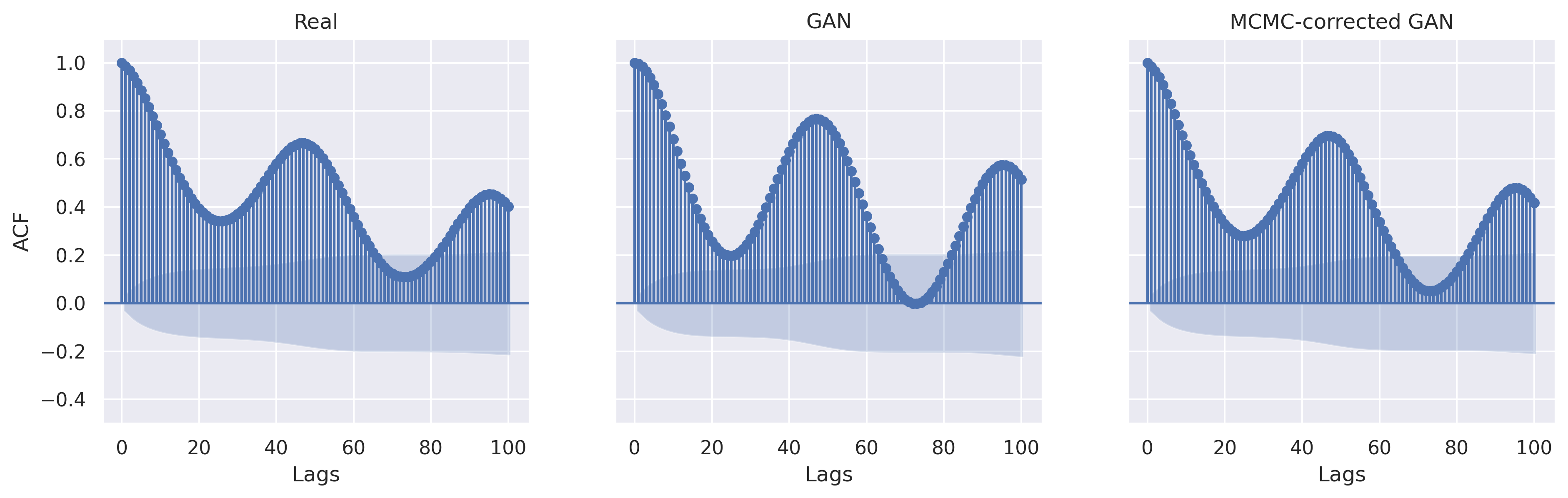}
	}
	\caption{ACF comparison illustrating temporal dependence preservation. The MCMC-refined models exhibit autocorrelation patterns more closely aligned with those of the real data.}
	\label{fig:acf}
\end{figure}

\begin{figure*}[htbp]
	\centering
	\subfigure
	{\label{fig:aecgan_licor_tsne}}
	\addtocounter{subfigure}{-1}
	\subfigure[t-SNE of Licor Dataset Generated by AECGAN]{
		\includegraphics[width=0.21\textwidth]{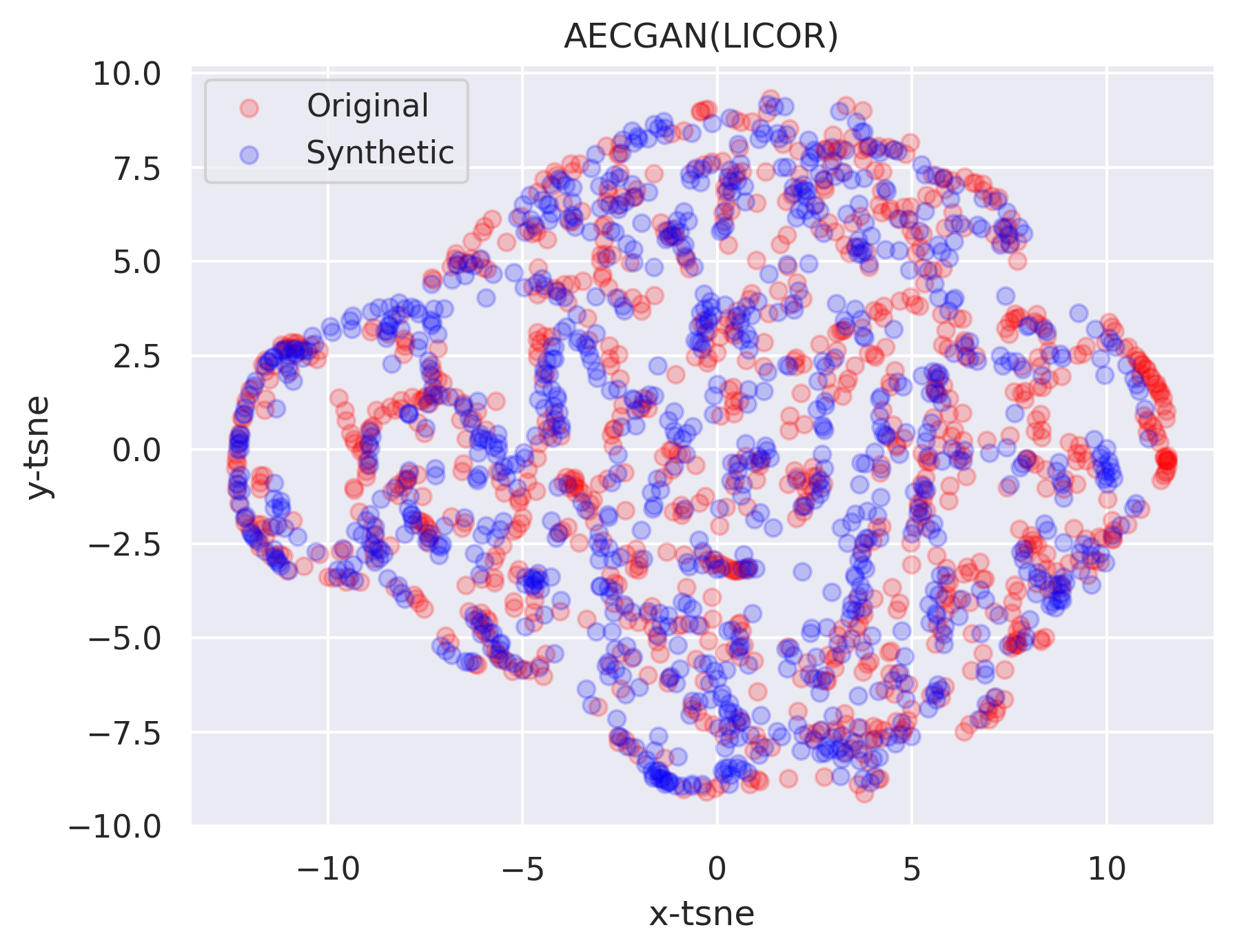}
	}
	\subfigure
	{\label{fig:aecgan_mcmc_licor_tsne}}
	\addtocounter{subfigure}{-1}
	\subfigure[t-SNE of Licor Dataset Generated by AECGAN-MCMC]{
		\includegraphics[width=0.21\textwidth]{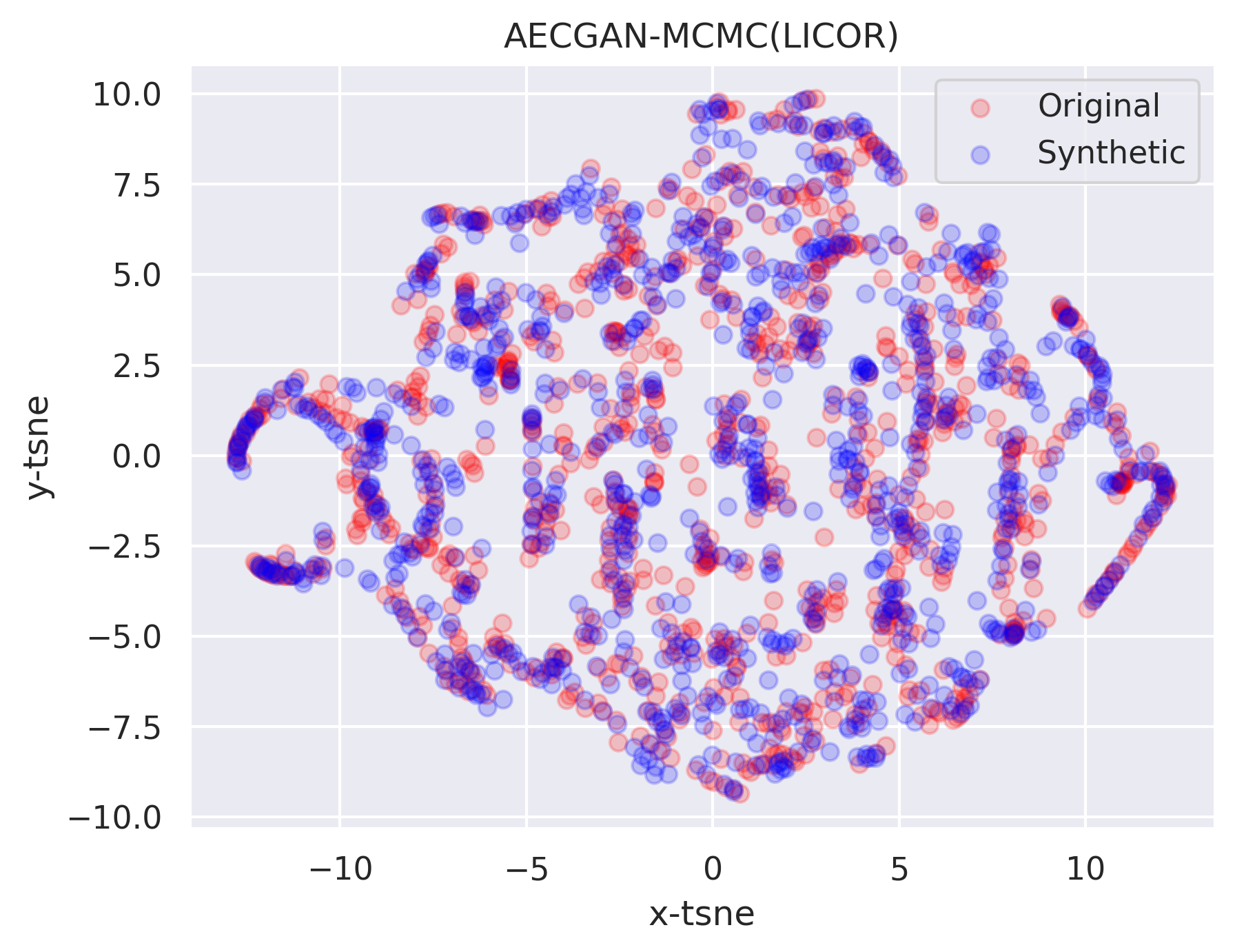}
	}
	\centering
	\subfigure
	{\label{fig:timegan_lorenz_tsne}}
	\addtocounter{subfigure}{-1}
	\subfigure[t-SNE of Lorenz Dataset Generated by TimeGAN]{
		\includegraphics[width=0.21\textwidth]{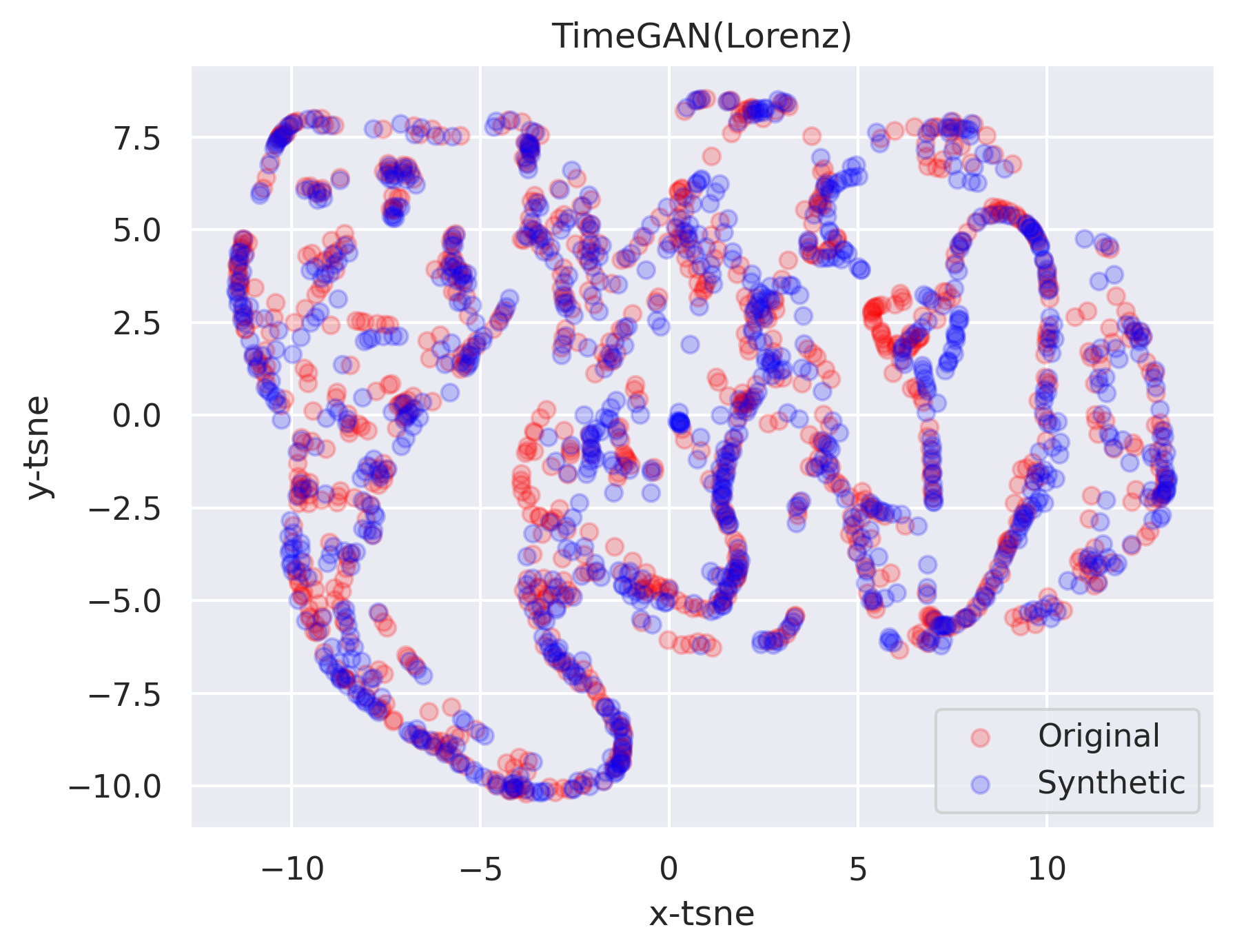}
	}
	\subfigure
	{\label{fig:timegan_mcmc_lorenz_tsne}}
	\addtocounter{subfigure}{-1}
	\subfigure[t-SNE of Lorenz Dataset Generated by TimeGAN-MCMC]{
		\includegraphics[width=0.21\textwidth]{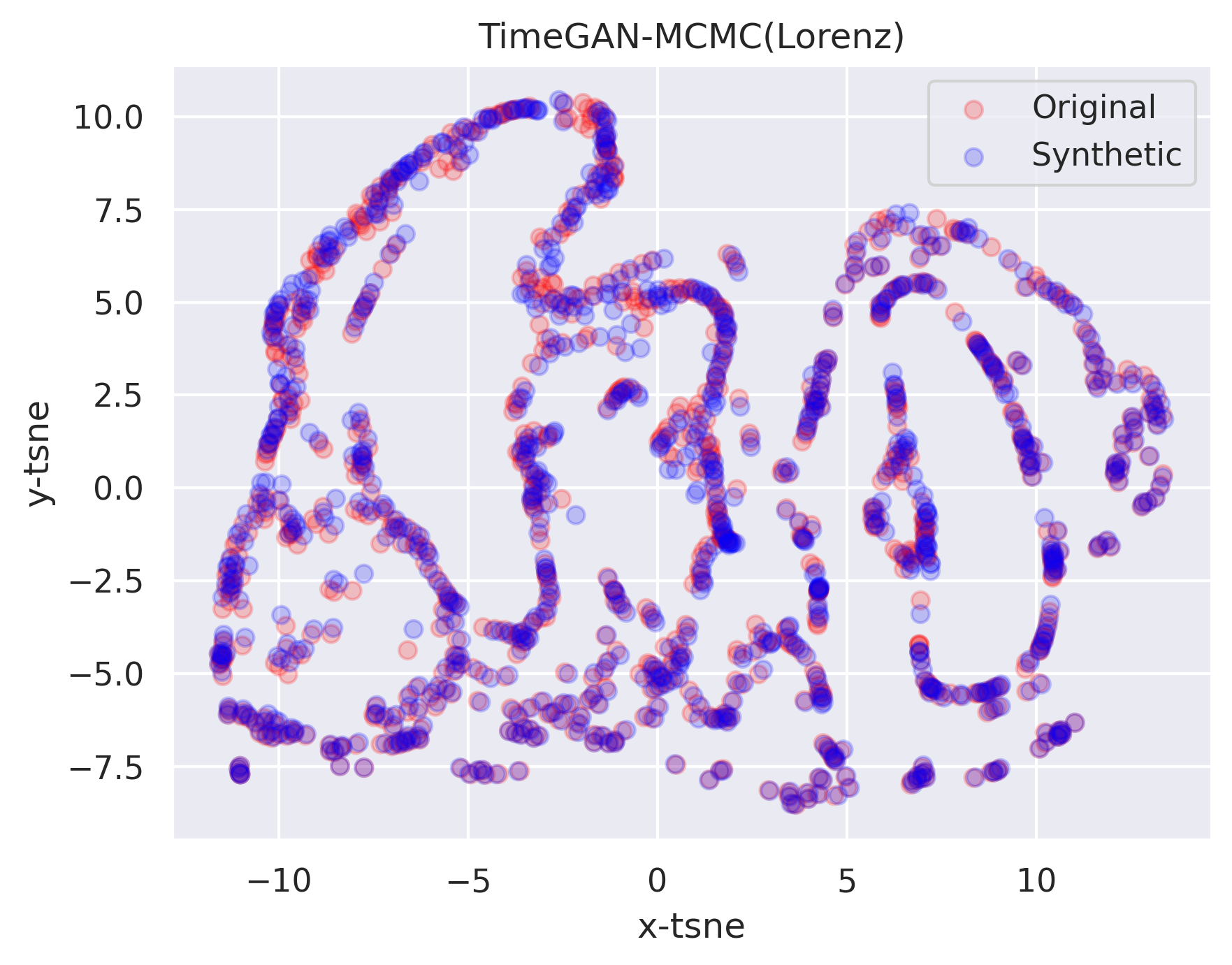}
	}
	\caption{t-SNE of Datasets Generated by GAN and Corresponding GAN-MCMC}
	\label{fig:all_tsne}
\end{figure*}

\begin{figure*}[htbp]
	\centering
	\subfigure
	{\label{fig:rcwgan_etth_pca}}
	\addtocounter{subfigure}{-1}
	\subfigure[PCA of ETTh Dataset Generated by RCWGAN]{
		\includegraphics[width=0.21\textwidth]{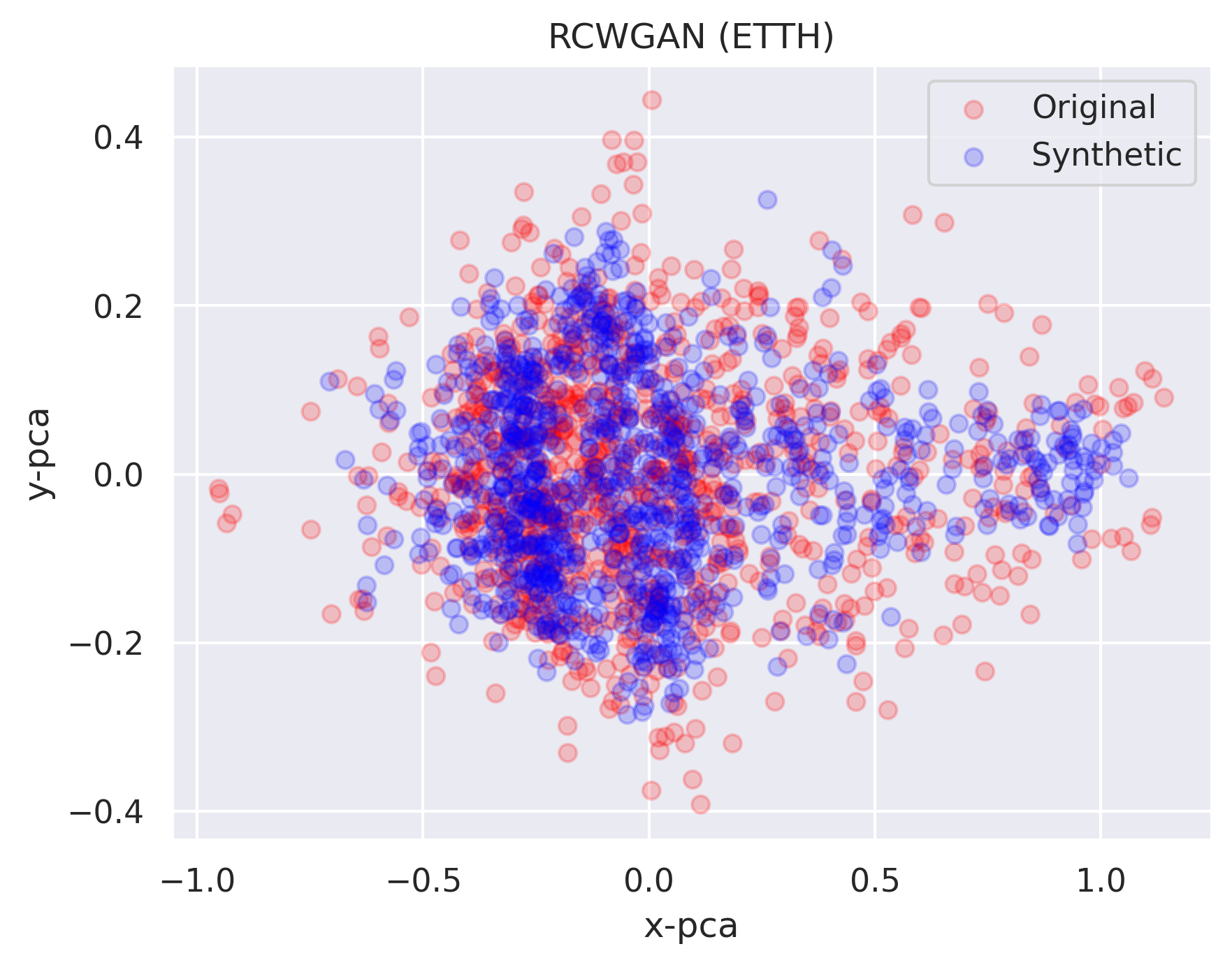}
	}
	\subfigure
	{\label{fig:rcwgan_mcmc_etth_pca}}
	\addtocounter{subfigure}{-1}
	\subfigure[PCA of ETTh Dataset Generated by RCWGAN-MCMC]{
		\includegraphics[width=0.21\textwidth]{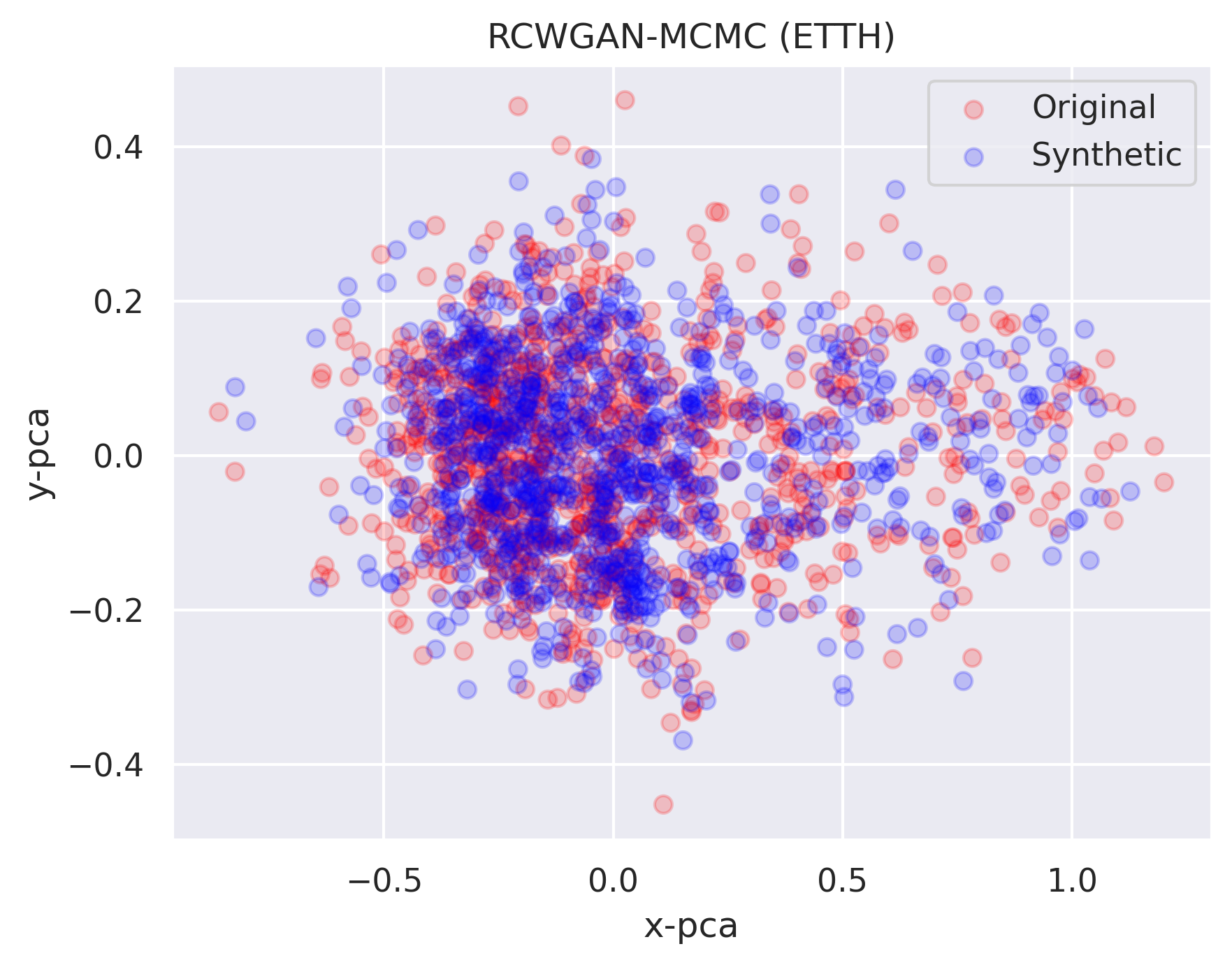}
	}
	\centering
	\subfigure
	{\label{fig:sigcwgan_ili_pca}}
	\addtocounter{subfigure}{-1}
	\subfigure[PCA of ILI Dataset Generated by SigCWGAN]{
		\includegraphics[width=0.21\textwidth]{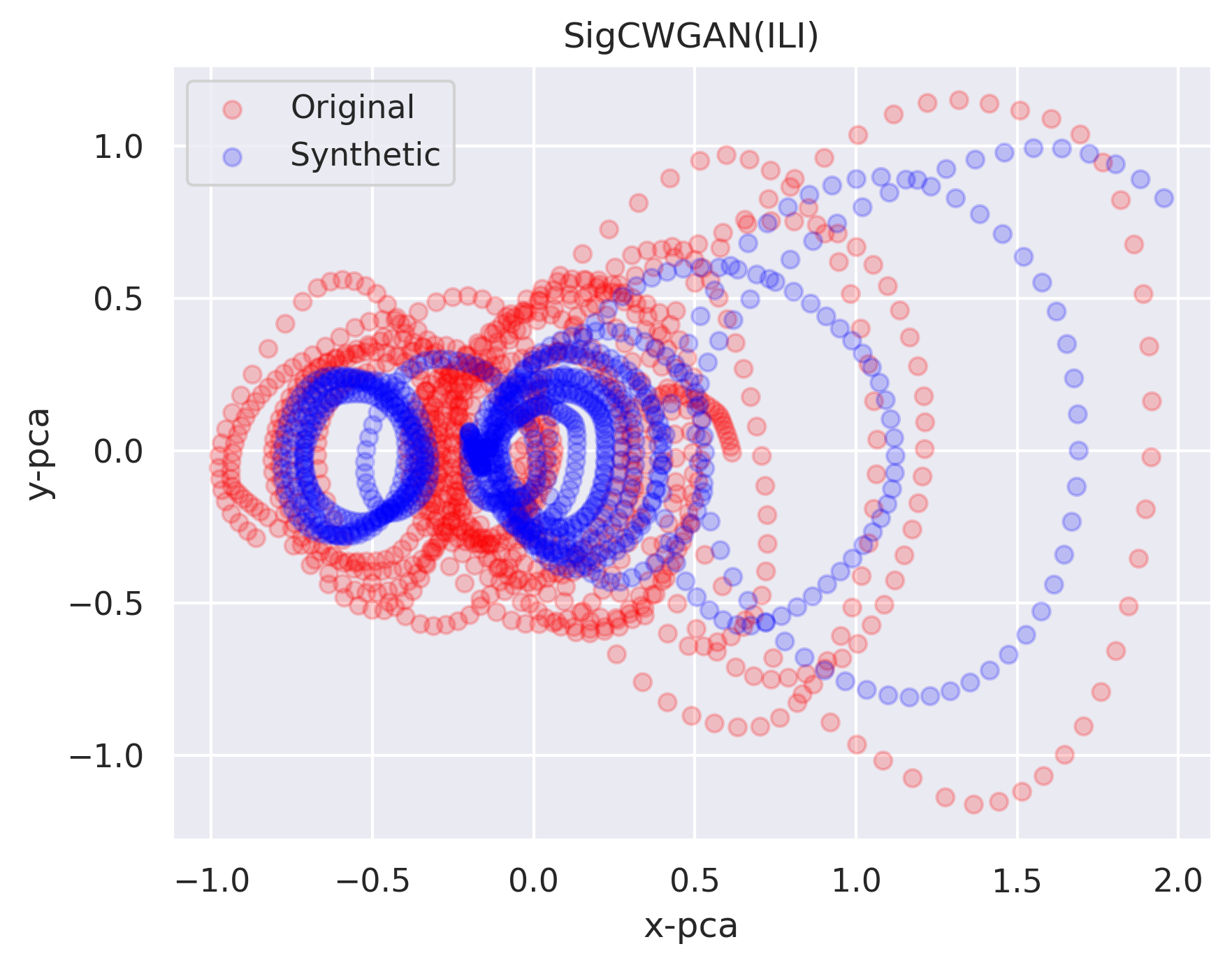}
	}
	\subfigure
	{\label{fig:siggcwgan_mcmc_ili_pca}}
	\addtocounter{subfigure}{-1}
	\subfigure[PCA of ILI Dataset Generated by SigCWGAN-MCMC]{
		\includegraphics[width=0.21\textwidth]{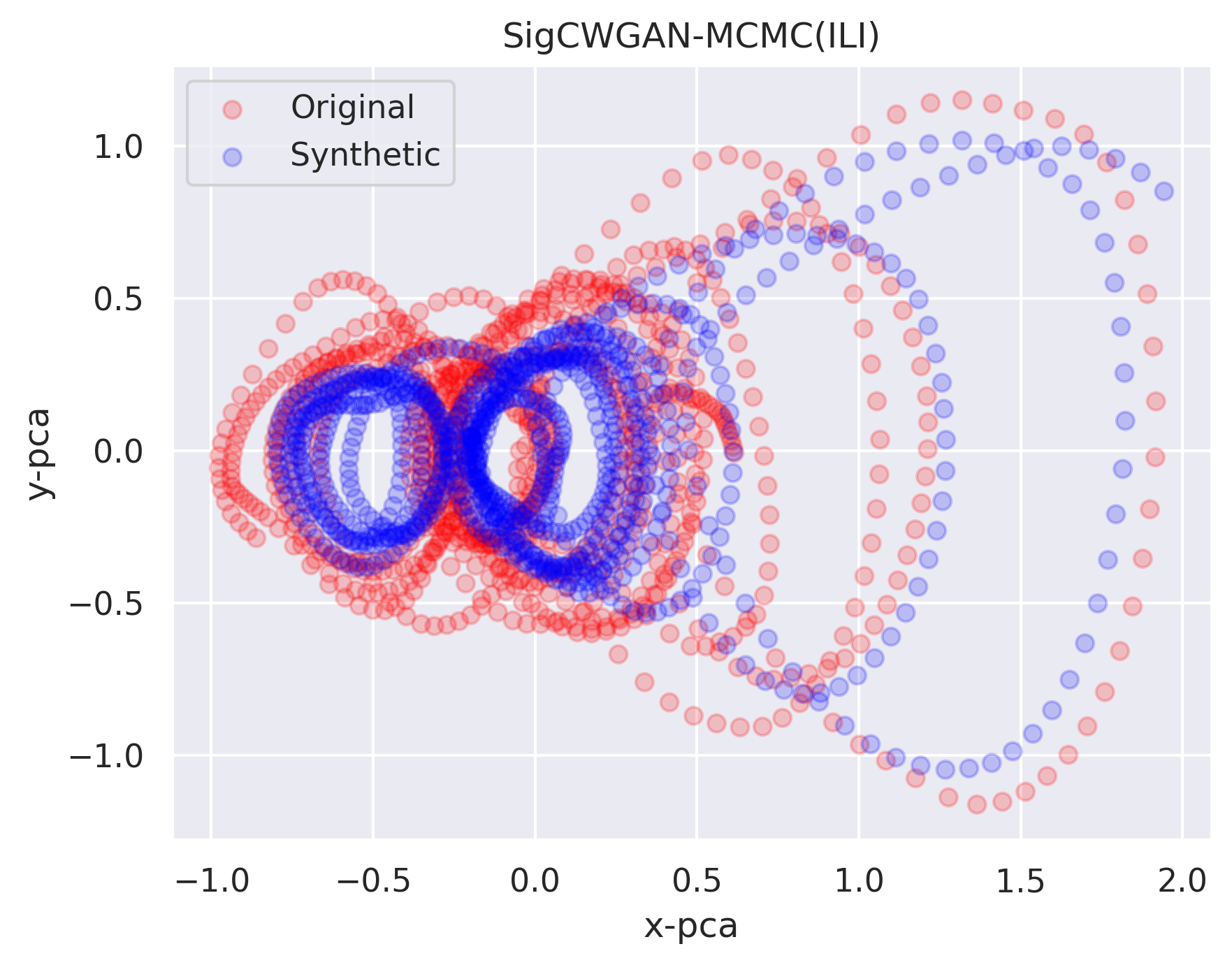}
	}
	\caption{PCA of Datasets Generated by GAN and Corresponding GAN-MCMC}
	\label{fig:all_pca}
\end{figure*}
To provide an intuitive evaluation of temporal dependence preservation, we compare the ACF of the original data, the sequences generated by the baseline GAN, and those generated by the proposed GAN-MCMC framework.

The autocorrelation function measures the linear dependence between observations separated by different time lags, and therefore reflects the intrinsic temporal memory structure of a time series. Preserving ACF patterns is crucial in regression-oriented time-series modeling, as it directly relates to the consistency of transition dynamics.

Figure \ref{fig:acf} illustrates the ACF curves computed for the ETTh and Licor dataset and the corresponding synthetic datasets generated by AEC-GAN-based and RC-GAN-based models. It can be observed that the ACF of sequences generated by the baseline GAN deviates noticeably from that of the real data. In contrast, the GAN-MCMC generated sequences exhibit autocorrelation patterns that closely align with those of the original time series across a wide range of lags.

This qualitative result, as shown in Table \ref{table:model_comparision_on_Lorenz}, \ref{table:model_comparision_on_Licor}, \ref{table:model_comparision_on_etth}, \ref{table:model_comparision_on_ili}, suggests that the proposed MCMC correction module effectively corrects temporal inconsistencies introduced during autoregressive generation, leading to improved preservation of the underlying dependence structure.

\subsection{Distributional Visualization via t-SNE and PCA}

In addition to ACF analysis, we further evaluate the distributional alignment between real and synthetic time series using t-distributed Stochastic Neighbor Embedding (t-SNE) and Principal Component Analysis (PCA). These visualization techniques provide complementary perspectives on the global structure of generated data in a lower-dimensional representation space.

\paragraph{t-SNE Visualization.}
We first apply t-SNE to compare the real data and synthetic samples generated by AEC-GAN and TimeGAN, along with their corresponding MCMC-refined variants. Specifically, the Licor dataset is used for evaluating AEC-GAN and AEC-GAN-MCMC, while the Lorenz dataset is employed for TimeGAN and TimeGAN-MCMC.

As shown in Figure \ref{fig:all_tsne}, , the synthetic data generated by the baseline GAN models exhibit noticeable deviations from the real data distribution in the embedded space. In contrast, the GAN-MCMC refined samples display significantly improved overlap with the real data clusters. This improvement suggests that the correction mechanism reduces divergence between empirical and generated distributions.

\paragraph{PCA Visualization.}
To further assess global structural consistency, we perform PCA analysis on synthetic sequences generated by RC-WGAN and SigCWGAN and their MCMC-refined counterparts. The ETTh dataset is used for evaluating RC-WGAN and RC-WGAN-MCMC, while the ILI dataset is adopted for SigCWGAN and SigCWGAN-MCMC.

Figure \ref{fig:all_pca}, illustrates the first two principal components of real and generated sequences. Compared with the baseline models, the MCMC-refined versions exhibit distributions that more closely align with the principal directions of the original data. The closer alignment in principal component space indicates that the MCMC module mitigates distortion in dominant variance directions.

Overall, the t-SNE and PCA visualizations consistently demonstrate that the proposed GAN-MCMC framework improves global distributional fidelity across multiple architectures and datasets.

\subsection{Quantitative Performance Comparison}

\begin{table*}[htbp]
	\centering
	\begin{threeparttable}
		\caption{Performance Comparison among Different Models on Generating Lorenz Data}
		\begin{tabular}
			{m{3cm}m{1.4cm}m{1.4cm}m{1.4cm}m{1.4cm}m{1.5cm}m{1.4cm}}
			\toprule & ACF & Skew & Kurt & R$^2$ & Discriminative & Predictive \\
			
			\midrule RCGAN  & 0.028$\pm$0.005  & 0.026$\pm$0.005  & 0.033$\pm$0.022  & 0.994$\pm$0.002  & 0.007$\pm$0.003  & 0.069$\pm$0.007 \\
			\midrule RCGAN-MCMC  & \textbf{0.017$\pm$0.009}  & \textbf{0.008$\pm$0.004}  & \textbf{0.011$\pm$0.007}  & \textbf{0.999$\pm$0.000}  & \textbf{0.002$\pm$0.001}  & \textbf{0.062$\pm$0.001} \\
			
			\midrule RCWGAN & 0.050$\pm$0.001 & 0.008$\pm$0.001 & 0.027$\pm$0.002 & 0.995$\pm$0.000 & 0.014$\pm$0.000 & 0.066$\pm$0.002 \\
			\midrule RCWGAN-MCMC & \textbf{0.015$\pm$0.002}  & \textbf{0.008$\pm$0.001} & \textbf{0.009$\pm$0.003} & \textbf{0.998$\pm$0.000} & \textbf{0.001$\pm$0.001} & \textbf{0.063$\pm$0.002} \\
			
			\midrule TimeGAN  & 0.031$\pm$0.015  & 0.015$\pm$0.007  & 0.040$\pm$0.009  & 0.994$\pm$0.003  & 0.013$\pm$0.008  & 0.067$\pm$0.006 \\
			\midrule TimeGAN-MCMC  & \textbf{0.015$\pm$0.002}  & \textbf{0.007$\pm$0.004}  & \textbf{0.016$\pm$0.004}  & \textbf{0.998$\pm$0.000}  & \textbf{0.004$\pm$0.002}  & \textbf{0.062$\pm$0.002} \\
			
			\midrule SigCWGAN  & 0.107$\pm$0.001  & 0.045$\pm$0.001  & 0.063$\pm$0.001  & 0.996$\pm$0.001  & 0.003$\pm$0.001  & 0.066$\pm$0.002 \\
			\midrule SigCWGAN-MCMC & \textbf{0.033$\pm$0.002}  & \textbf{0.017$\pm$0.001}  & \textbf{0.018$\pm$0.001}  & \textbf{0.998$\pm$0.001}  & \textbf{0.001$\pm$0.001}  & \textbf{0.064$\pm$0.001} \\
			
			\midrule AECGAN & 0.081$\pm$0.021 & 0.035$\pm$0.006 & 0.045$\pm$0.009 & 0.989$\pm$0.004 & 0.014$\pm$0.007 & 0.070$\pm$0.007 \\
			\midrule AECGAN-MCMC & \textbf{0.038$\pm$0.013} & \textbf{0.016$\pm$0.008} & \textbf{0.029$\pm$0.014} & \textbf{0.996$\pm$0.003} & \textbf{0.004$\pm$0.003} & \textbf{0.064$\pm$0.003} \\
			
			\bottomrule
		\end{tabular}
		\label{table:model_comparision_on_Lorenz}
	\end{threeparttable}
\end{table*}

\begin{table*}[htbp]
	\centering
	\begin{threeparttable}
		\caption{Performance Comparison among Different Models on Generating Licor Data}
		\begin{tabular}
			{m{3cm}m{1.4cm}m{1.4cm}m{1.4cm}m{1.4cm}m{1.5cm}m{1.4cm}}
			\toprule & ACF & Skew & Kurt & R$^2$ & Discriminative & Predictive \\
			
			\midrule RCGAN  & 0.353$\pm$0.098  & 0.118$\pm$0.025  & 0.300$\pm$0.081  & 0.798$\pm$0.033  & 0.045$\pm$0.034  & 0.086$\pm$0.001 \\
			\midrule RCGAN-MCMC  & \textbf{0.147$\pm$0.039}  & \textbf{0.063$\pm$0.019}  & \textbf{0.203$\pm$0.083}  & \textbf{0.930$\pm$0.006}  & \textbf{0.008$\pm$0.003}  & \textbf{0.085$\pm$0.000} \\
			
			\midrule RCWGAN & 0.427$\pm$0.009 & 0.119$\pm$0.002 & 0.238$\pm$0.004 & 0.762$\pm$0.001 & 0.092$\pm$0.001 & 0.084$\pm$0.000 \\
			\midrule RCWGAN-MCMC & \textbf{0.085$\pm$0.005} & \textbf{0.050$\pm$0.003} & \textbf{0.102$\pm$0.014} & \textbf{0.926$\pm$0.001} & \textbf{0.013$\pm$0.002} & \textbf{0.084$\pm$0.000} \\
			
			\midrule TimeGAN  & 0.343$\pm$0.099  & 0.132$\pm$0.030  & 0.311$\pm$0.080  & 0.785$\pm$0.070  & 0.055$\pm$0.058  & 0.085$\pm$0.001 \\
			\midrule TimeGAN-MCMC  & \textbf{0.121$\pm$0.012}  & \textbf{0.053$\pm$0.015}  & \textbf{0.173$\pm$0.082}  & \textbf{0.933$\pm$0.007}  & \textbf{0.008$\pm$0.005}  & \textbf{0.084$\pm$0.001} \\
			
			\midrule SigCWGAN & 0.843$\pm$0.004 & 0.461$\pm$0.002 & 0.692$\pm$0.004 & 0.689$\pm$0.001 & 0.038$\pm$0.003 & 0.087$\pm$0.000 \\
			\midrule SigCWGAN-MCMC & \textbf{0.440$\pm$0.011} & \textbf{0.341$\pm$0.007} & \textbf{0.507$\pm$0.011} & \textbf{0.794$\pm$0.005} & \textbf{0.002$\pm$0.001} & \textbf{0.084}$\pm$\textbf{0.000} \\
			
			\midrule AECGAN & 0.443$\pm$0.134 & 0.141$\pm$0.028 & 0.301$\pm$0.077 & 0.790$\pm$0.020 & 0.058$\pm$0.017 & 0.085$\pm$0.001 \\
			\midrule AECGAN-MCMC & \textbf{0.261$\pm$0.063} & \textbf{0.081$\pm$0.011} & \textbf{0.237$\pm$0.045} & \textbf{0.906$\pm$0.003} & \textbf{0.014$\pm$0.007} & \textbf{0.084$\pm$0.000} \\
			
			\bottomrule
		\end{tabular}
		\label{table:model_comparision_on_Licor}
	\end{threeparttable}
\end{table*}

\begin{table*}[htbp]
	\centering
	\begin{threeparttable}
		\caption{Performance Comparison among Different Models on Generating ETTH Data}
		\begin{tabular}
			{m{3cm}m{1.4cm}m{1.4cm}m{1.4cm}m{1.4cm}m{1.5cm}m{1.4cm}}
			\toprule & ACF & Skew & Kurt & R$^2$ & Discriminative & Predictive \\
			
			\midrule RCGAN  & 0.986$\pm$0.077  & 0.498$\pm$0.137  & 3.221$\pm$0.913  & 0.616$\pm$0.037  & 0.167$\pm$0.037  & 0.118$\pm$0.003 \\
			\midrule RCGAN-MCMC  & \textbf{0.299$\pm$0.041}  & \textbf{0.088$\pm$0.023}  & \textbf{0.368$\pm$0.232}  & \textbf{0.743$\pm$0.014}  & \textbf{0.069$\pm$0.050}  & \textbf{0.110$\pm$0.002} \\
			
			\midrule RCWGAN & 1.095$\pm$0.002 & 0.334$\pm$0.005 & 1.692$\pm$0.035 & 0.602$\pm$0.002 & 0.227$\pm$0.001 & 0.115$\pm$0.002 \\
			\midrule RCWGAN-MCMC & \textbf{0.231$\pm$0.004} & \textbf{0.109$\pm$0.006} & \textbf{0.152$\pm$0.030} & \textbf{0.682$\pm$0.001} & \textbf{0.075$\pm$0.003} & \textbf{0.108$\pm$0.001} \\
			
			\midrule TimeGAN  & 0.916$\pm$0.117  & 0.390$\pm$0.121  & 2.208$\pm$0.647  & 0.630$\pm$0.032  & 0.173$\pm$0.061  & 0.122$\pm$0.008 \\
			\midrule TimeGAN-MCMC  & \textbf{0.233$\pm$0.054}  & \textbf{0.102$\pm$0.014}  & \textbf{0.353$\pm$0.101}  & \textbf{0.758$\pm$0.016}  & \textbf{0.055$\pm$0.044}  & \textbf{0.110$\pm$0.002} \\
			
			\midrule SigCWGAN & 2.086$\pm$0.005 & 1.187$\pm$0.003 & 4.100$\pm$0.027 & 0.521$\pm$0.001 & 0.113$\pm$0.003 & \textbf{0.120$\pm$0.003} \\
			\midrule SigCWGAN-MCMC & \textbf{0.834$\pm$0.008} & \textbf{0.786$\pm$0.004} & \textbf{3.094$\pm$0.029} & \textbf{0.618$\pm$0.012} & \textbf{0.018$\pm$0.002} & 0.122$\pm$0.004 \\
			
			\midrule AECGAN & 0.916$\pm$0.138 & 0.242$\pm$0.074 & 1.088$\pm$0.364 & 0.577$\pm$0.051 & 0.191$\pm$0.091 & 0.126$\pm$0.009 \\
			\midrule AECGAN-MCMC & \textbf{0.400$\pm$0.088} & \textbf{0.134$\pm$0.026} & \textbf{0.465$\pm$0.202} & \textbf{0.789$\pm$0.014} & \textbf{0.080$\pm$0.052} & \textbf{0.109$\pm$0.003} \\
			\bottomrule
		\end{tabular}
		\label{table:model_comparision_on_etth}
	\end{threeparttable}
\end{table*}

\begin{table*}[htbp]
	\centering
	\begin{threeparttable}
		\caption{Performance Comparison among Different Models on Generating ILI Data}
		\begin{tabular}
			{m{3cm}m{1.4cm}m{1.4cm}m{1.4cm}m{1.4cm}m{1.5cm}m{1.4cm}}
			\toprule & ACF & Skew & Kurt & R$^2$ & Discriminative & Predictive \\
			
			\midrule RCGAN  & 0.756$\pm$0.193  & 0.415$\pm$0.135  & 3.105$\pm$2.013  & 0.694$\pm$0.059  & 0.070$\pm$0.039  & 0.033$\pm$0.003 \\
			\midrule RCGAN-MCMC  & \textbf{0.576$\pm$0.142}  & \textbf{0.240$\pm$0.111}  & \textbf{1.562$\pm$0.765}  & \textbf{0.800$\pm$0.046}  & \textbf{0.026$\pm$0.014}  & \textbf{0.032$\pm$0.001} \\
			
			\midrule RCWGAN & 0.531$\pm$0.023  & 0.226$\pm$0.011  & \textbf{0.989$\pm$0.089}  & 0.797$\pm$0.005  & 0.052$\pm$0.002  & 0.0310$\pm$0.001 \\
			\midrule RCWGAN-MCMC & \textbf{0.183$\pm$0.017} & \textbf{0.207$\pm$0.042} & 1.286$\pm$0.362 & \textbf{0.892$\pm$0.002} & \textbf{0.008$\pm$0.004} & \textbf{0.030$\pm$0.001} \\
			
			\midrule TimeGAN  & 0.820$\pm$0.054  & 0.420$\pm$0.200  & 2.388$\pm$1.801  & 0.622$\pm$0.214  & 0.052$\pm$0.019  & 0.036$\pm$0.002 \\
			\midrule TimeGAN-MCMC  & \textbf{0.634$\pm$0.118}  & \textbf{0.380$\pm$0.237}  & \textbf{2.251$\pm$2.203}  & \textbf{0.718$\pm$0.217}  & \textbf{0.024$\pm$0.019}  & \textbf{0.034$\pm$0.001} \\
			
			\midrule SigCWGAN  & 1.322$\pm$0.007  & 0.258$\pm$0.004  & 1.028$\pm$0.026  & 0.471$\pm$0.010  & 0.143$\pm$0.002  & 0.034$\pm$0.000 \\
			\midrule SigCWGAN-MCMC & \textbf{1.006$\pm$0.005}  & \textbf{0.182$\pm$0.003}  & \textbf{0.984$\pm$0.012}  & \textbf{0.556$\pm$0.002}  & \textbf{0.084$\pm$0.007}  & \textbf{0.033$\pm$0.000} \\
			
			\midrule AECGAN & 0.936$\pm$0.088 & 0.601$\pm$0.160 & 3.431$\pm$1.097 & 0.617$\pm$0.040 & 0.035$\pm$0.018 & 0.036$\pm$0.005 \\
			\midrule AECGAN-MCMC & \textbf{0.740$\pm$0.098} & \textbf{0.405$\pm$0.093} & \textbf{2.456$\pm$0.529} & \textbf{0.750$\pm$0.047} & \textbf{0.020$\pm$0.022} & \textbf{0.032$\pm$0.002} \\
			
			\bottomrule
		\end{tabular}
		\label{table:model_comparision_on_ili}
	\end{threeparttable}
\end{table*}

Tables \ref{table:model_comparision_on_Lorenz}, \ref{table:model_comparision_on_Licor}, \ref{table:model_comparision_on_etth}, \ref{table:model_comparision_on_ili}  present a comprehensive quantitative comparison between baseline GAN models and the corresponding GAN-MCMC framework across four representative datasets: Lorenz, Licor, ETTh, and ILI. The evaluation metrics jointly assess temporal dependence (ACF), higher-order distributional characteristics (skewness error and kurtosis error), global similarity ($R^2$), realism (discriminative score), and predictive utility (predictive score). The better-performing results are shown in bold.

As shown in Table \ref{table:model_comparision_on_Lorenz}, on the Lorenz dataset, the proposed MCMC correction consistently improves temporal and distributional metrics across all GAN architectures. In particular, ACF error is substantially reduced for RCGAN, RCWGAN, TimeGAN, SigCWGAN, and AECGAN after MCMC correction. Skewness and kurtosis errors are also significantly lowered, indicating improved preservation of higher-order statistics. Moreover, $R^2$ values increase toward near-perfect alignment, while discriminative scores approach zero, suggesting that corrected samples are more difficult to distinguish from real data. These results demonstrate that the MCMC module effectively enhances both temporal fidelity and statistical realism in chaotic dynamical systems.

As shown in Table \ref{table:model_comparision_on_Licor}, the Licor dataset, which exhibits more complex real-world environmental dynamics, shows even more pronounced improvements. The baseline GAN models suffer from large ACF error, skewness error, and kurtosis error. After MCMC correction, ACF errors are dramatically reduced (e.g., RCWGAN and TimeGAN), and $R^2$ scores increase significantly. Discriminative scores consistently decrease across models, indicating improved temporal dynamics in synthetic time series dataset. Notably, predictive scores remain stable or slightly improved, suggesting that temporal dynamics are preserved without sacrificing forecasting utility. These results confirm that the proposed correction mechanism effectively mitigates distributional deviation in practical regression-oriented datasets.

As shown in Table \ref{table:model_comparision_on_etth}, on the ETTh dataset, which contains long-term periodic and trend components, the GAN-MCMC models show substantial reductions in ACF, skewness, and kurtosis errors across all GANs. In particular, the large ACF discrepancies observed in baseline models (e.g., RCWGAN and SigCWGAN) are significantly reduced after correction. Although improvements in $R^2$ are moderate, the consistent decrease in discriminative scores indicates improved alignment between generated and real distributions. These findings suggest that the correction process effectively corrects temporal inconsistencies in long-horizon forecasting scenarios.

As shown in Table \ref{table:model_comparision_on_ili}, for the ILI dataset, which features seasonal epidemiological patterns, MCMC correction again leads to systematic improvements in ACF and higher-order moment errors, including skewness error and kurtosis error. While gains in $R^2$ and predictive score are dataset-dependent, the consistent reduction in discriminative scores across models demonstrates improved statistical realism. The correction process thus enhances both local temporal dynamics and global distributional consistency in seasonal time-series data.

Across all datasets and GAN architectures, the GAN-MCMC models consistently reduce ACF, skewness, and kurtosis errors while improving discriminative realism without degrading predictive performance. The reduction in ACF error indicates improved preservation of second-order temporal dependencies, suggesting that the correction mitigates temporal drift introduced during autoregressive generation. Improvements in skewness and kurtosis further demonstrate enhanced alignment of higher-order distributional characteristics, indicating that the MCMC module stabilizes global distributional structure beyond local transitions. The consistent decrease in discriminative scores shows that refined samples are more indistinguishable from real data, and the robustness of these improvements across architectures suggests that the correction mechanism is model-agnostic. Importantly, predictive performance remains stable or improves slightly, indicating that enhanced realism does not compromise forecasting utility. Overall, these results demonstrate that distribution shift in autoregressive generative models can be effectively mitigated through post-generation correction, reinforcing the importance of preserving conditional transition dynamics rather than merely matching marginal distributions in time-series augmentation.

\section{Conclusion and Future Works}\label{sec:conclusion_future_work}

In this paper, we investigated the problem of distribution shift in autoregressive GAN-based time-series generation. While existing generative models primarily focus on matching marginal data distributions, we argue that preserving temporal dynamics is fundamentally more important for regression-oriented time-series modeling. Autoregressive generation inevitably introduces compounding deviations, leading to temporal drift and distributional distortion.

To address this challenge, we proposed a model-agnostic MCMC-based correction framework that refines generated trajectories by enforcing consistency with empirical transition statistics of the original multivariate time series. Our theoretical analysis demonstrated how distribution shift arises in conditional GANs (CGANs) and how MCMC correction can mitigate such deviations. Extensive experiments on the Lorenz, Licor, ETTh, and ILI datasets consistently showed that the proposed GAN-MCMC framework improves ACF error, skewness error, kurtosis error, $R^2$, discriminative score, and predictive score.

These findings suggest that high-quality time-series augmentation requires not only adversarial distribution matching but also explicit enforcement of temporal dynamical consistency. By shifting the focus from marginal fidelity to the preservation of temporal dynamics, this work provides a principled direction for improving synthetic multivariate time series generated by CGANs.

Future work may explore integrating the MCMC mechanism into diffusion models and addressing distribution shift in data generated by diffusion models. Furthermore, we plan to investigate integrating the MCMC framework directly into the training process of generative models.

\bibliographystyle{unsrt} 
\bibliography{references} 


\end{document}